\documentclass{article}

    \PassOptionsToPackage{numbers, compress}{natbib}
    \usepackage[preprint]{neurips_2020}

\usepackage[utf8]{inputenc} % allow utf-8 input
\usepackage[T1]{fontenc}    % use 8-bit T1 fonts
\usepackage{hyperref}       % hyperlinks
\usepackage{url}            % simple URL typesetting
\usepackage{booktabs}       % professional-quality tables
\usepackage{amsfonts}       % blackboard math symbols
\usepackage{nicefrac}       % compact symbols for 1/2, etc.
\usepackage{microtype}      % microtypography

\usepackage{wrapfig}
\usepackage{graphicx}
\usepackage{enumitem}
\usepackage{amsthm}
\usepackage{amsmath}
\usepackage{subfigure}
\newtheorem*{remark}{Remark}
\usepackage{multirow,color}
\usepackage{algorithm}
\usepackage{algorithmic}

\newtheorem{theorem}{Theorem}[section]

\newtheorem{lemma}[theorem]{Lemma}
\theoremstyle{definition}
\newtheorem{definition}[theorem]{Definition}

\newcommand{\model}{{{TransLATE}}}

\title{Continuous Transfer Learning \\ with Label-informed Distribution Alignment}

\author{%
  Jun Wu \\
  University of Illinois at Urbana-Champaign\\
  \texttt{junwu3@illinois.edu}
  \And
  Jingrui He \\
  University of Illinois at Urbana-Champaign\\
  \texttt{jingrui@illinois.edu}
}
\begin{document}

\maketitle

\vspace{-7mm}
\begin{abstract}
Transfer learning has been successfully applied across many high-impact applications. However, most existing work focuses on the static transfer learning setting, and very little is devoted to modeling the time evolving target domain, such as the online reviews for movies. To bridge this gap, in this paper, we study a novel continuous transfer learning setting with a time evolving target domain. One major challenge associated with continuous transfer learning is the potential occurrence of negative transfer as the target domain evolves over time. To address this challenge, we propose a novel label-informed $\mathcal{C}$-divergence between the source and target domains in order to measure the shift of data distributions as well as to identify potential negative transfer. We then derive the error bound for the target domain using the empirical estimate of our proposed $\mathcal{C}$-divergence. Furthermore, we propose a generic adversarial Variational Auto-encoder framework named \model{} by minimizing the classification error and $\mathcal{C}$-divergence of the target domain between consecutive time stamps in a latent feature space. In addition, we define a transfer signature for characterizing the negative transfer based on $\mathcal{C}$-divergence, which indicates that larger $\mathcal{C}$-divergence implies a higher probability of negative transfer in real scenarios. Extensive experiments on synthetic and real data sets demonstrate the effectiveness of our \model{} framework.
\end{abstract}

\vspace{-5mm}
\section{Introduction}
\vspace{-2mm}

\begin{wrapfigure}{r}{0.5\textwidth}
  \vspace{-7mm}
  \begin{center}
    \includegraphics[width=0.5\textwidth]{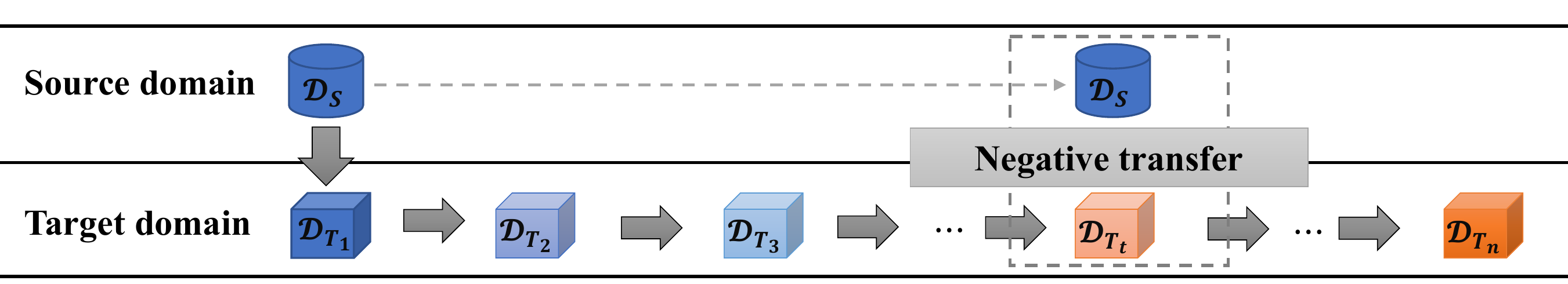}
  \end{center}
  \vspace{-5mm}
  \caption{Illustration of continuous transfer learning. It learns a predictive function in $\mathcal{D}_{T_t}$ using knowledge from both source domain $\mathcal{D}_S$ and historical target domain $\mathcal{D}_{T_i} (i=1,\cdots,t-1)$. Directly transferring from the source domain $\mathcal{D}_S$ to the target domain $\mathcal{D}_{T_t}$ might lead to negative transfer with undesirable predictive performance.}\label{fig:ctl}
  \vspace{-2mm}
\end{wrapfigure}
Transfer learning has achieved significant success across multiple high-impact application domains~\cite{pan2009survey}. Compared to conventional machine learning methods assuming both training and test data have the same data distribution, transfer learning allows us to learn the target domain with limited label information by leveraging a related source domain with abundant label information~\cite{sun2016return, ying2018transfer}. However, in many real applications, the target domain is constantly evolving over time. For example, the online movie reviews are changing over the years: some famous movies were not well received by the mainstream audience when they were first released, but became famous only years later (e.g., {\it Citizen Cane}, {\it Fight Club}, and {\it The Shawshank Redemption}); whereas the online book reviews typically do not have this type of dynamics. It is challenging to transfer knowledge from the static source domain (e.g., the book reviews) to the time evolving target domain (e.g., the movie reviews). Therefore, in this paper, we study a novel transfer learning setting with a static source domain and a continuously time evolving target domain (see Fig. \ref{fig:ctl}). The unique challenge for continuous transfer learning lies in the time evolving nature of the task relatedness between the static source domain and the time evolving target domain. Although the change in the target data distribution in consecutive time stamps might be small, over time, the cumulative change in the target domain might even lead to negative transfer~\cite{rosenstein2005transfer}.

Existing theoretical analysis on transfer learning~\cite{ben2010theory, mansour2009domain} showed that the target error is typically bounded by the source error, the domain discrepancy and the difference of labeling functions. It has been observed~\cite{zhao2019learning, wu2019domain} that marginal feature distribution alignment might not guarantee the minimization of the target error in real world scenarios. In other words, in the context of continuous transfer learning, it would lead to the sub-optimal solution (or even negative transfer) with undesirable predictive performance when directly transferring from $\mathcal{D}_S$ to the target domain $\mathcal{D}_{T_t}$ at the $t^{\mathrm{th}}$ time stamp. This paper aims to bridge the gap in terms of both the theoretical analysis and the empirical solutions for the target domain with a time evolving distribution, which lead to a novel continuous transfer learning framework and the characterization of negative transfer. The main contributions of this paper are summarized as follows:
\begin{itemize}[leftmargin=*,noitemsep,nolistsep]
    \item \textbf{Theoretical results:} We propose a label-informed domain discrepancy measure ($\mathcal{C}$-divergence) with its empirical estimate, followed by the
    error bounds for both static and continuous transfer learning settings. Then we define a transfer signature for characterizing the negative transfer based on $\mathcal{C}$-divergence.
    \item \textbf{Framework:} We propose a generic continuous transfer learning framework (\model{}) using our proposed $\mathcal{C}$-divergence, which re-aligns the label-informed data distribution in a latent feature space for the target domain between consecutive time stamps.
    \item \textbf{Experiments:} Extensive experimental results on synthetic and real-world data sets confirm the effectiveness of the proposed \model{} framework.
\end{itemize}

\vspace{-2mm}
\section{Preliminaries}\label{Preliminaries}
\vspace{-2mm}
In this section, we introduce the notation and problem definition of continuous transfer learning.

\vspace{-2mm}
\subsection{Notation}
\vspace{-2mm}
We use $\mathcal{X}$ and $\mathcal{Y}$ to denote the data input space and label space. Let $\mathcal{D}_S$ and $\mathcal{D}_T$ denote the source and target domains with data distribution $p_S(\mathbf{x},y)$ and $p_T(\mathbf{x},y)$ over $\mathcal{X}\times \mathcal{Y}$, respectively. Let $\mathcal{H}$ be a hypothesis class on $\mathcal{X}$, where a hypothesis is a function $h:\mathcal{X} \rightarrow \mathcal{Y}$. For transfer learning, it assumes that there are $m_S$ labeled source examples drawn independently from $\mathcal{D}_S$ and $m_T$ labeled target examples drawn independently from $\mathcal{D}_T$. The notation is summarized in Table~\ref{tab: notations} in the appendices.

\vspace{-2mm}
\subsection{Continuous Transfer Learning}
\vspace{-2mm}
Transfer learning~\cite{pan2009survey} refers to the knowledge transfer from source domain to target domain such that the prediction performance on the target domain could be significantly improved as compared to learning from the target domain alone. However, in some applications, the target domain is changing over time, hence the time evolving relatedness between the source and target domains.
This motivates us to consider a novel transfer learning setting with the time evolving target domain. We formally define the continuous transfer learning problem as follows.
\begin{definition}({\em Continuous Transfer Learning})
    Given a source domain $\mathcal{D}_S$ (available at time stamp $t=1$) and a time evolving target domain $\{\mathcal{D}_{T_t}\}_{t=1}^n$ with time stamp $t$, {\it continuous transfer learning} aims to improve the prediction function for target domain $\mathcal{D}_{T_k}$ using the knowledge from source domain $\mathcal{D}_S$ and the historical target domain $\mathcal{D}_{T_t} (t=1,\cdots,k-1)$.
\end{definition}

\vspace{-2mm}
\section{Label-informed Domain Discrepancy}\label{Discrepancy}
\vspace{-2mm}
% In this section, we propose a label-informed domain discrepancy, followed by its empirical estimate.

\subsection{$\mathcal{C}$-divergence}\label{c_divergence}
\vspace{-2mm}
We begin by considering the binary classification setting, i.e., $\mathcal{Y}=\{0, 1\}$. The source error of a hypothesis $h$ can be defined as follows: $\epsilon_S(h) = \mathbb{E}_{(\mathbf{x},y)\sim p_S(\mathbf{x}, y)} \big[\mathcal{L}(h(\mathbf{x}), y) \big]$
where $\mathcal{L}(\cdot, \cdot)$ is the loss function.
Its empirical estimate is denoted as $\hat{\epsilon}_S(h)$. Similarly, we define the target error $\epsilon_T(h)$ and the empirical estimate of the target error $\hat{\epsilon}_T(h)$ over the target distribution $p_T(\mathbf{x}, y)$.

We then define a label-informed domain discrepancy using the following $L_1$ or variation divergence over joint distributions (i.e., $p_S(\mathbf{x}, y)$ for source domain $\mathcal{D}_S$ and $p_T(\mathbf{x}, y)$ for target domain $\mathcal{D}_T$) between data features and class label:
\begin{equation}\label{L_1_divergence}
    d_1(\mathcal{D}_S, \mathcal{D}_T) = \sup_{Q\in \mathcal{Q}} \big| \text{Pr}_{\mathcal{D}_S}[Q] - \text{Pr}_{\mathcal{D}_T}[Q] \big|
\end{equation}
where $\mathcal{Q}$ is the set of measurable subsets under $p_S(\mathbf{x}, y)$ and $p_T(\mathbf{x}, y)$. 
\begin{remark}
Compared with existing domain divergence measures~\cite{ben2010theory}, in our definition, every measurable subset involves both features and class labels, while existing work only considers the features (i.e., no label information used in the definition of existing domain divergence measures). The additional label information improves the discrimination of subset in $\mathcal{Q}$, thus leading to tighter domain discrepancy between source and target domains.
\end{remark}

For a hypothesis $h\in \mathcal{H}$, we denote $I(h)$ to be the subset of $\mathcal{X}$ such that $\mathbf{x}\in I(h) \Leftrightarrow h(\mathbf{x})=1$. In order to estimate the label-informed domain discrepancy from finite samples in practice, instead of Eq.~(\ref{L_1_divergence}), we propose the following $\mathcal{C}$-divergence between $\mathcal{D}_S$ and $\mathcal{D}_T$, taking into consideration the joint distribution between features and class labels:
\begin{equation}\label{definition_divergence}
% \begin{align}
d_{\mathcal{C}}(\mathcal{D}_S, \mathcal{D}_T) = \sup_{h\in \mathcal{H}} \Big|\text{Pr}_{\mathcal{D}_S}[\{I(h), y=1\}\cup \{\overline{I(h)}, y=0\}] - \text{Pr}_{\mathcal{D}_T}[\{I(h), y=1\}\cup \{\overline{I(h)}, y=0\}] \Big|
% \end{align}
\end{equation}
where $\overline{I(h)}$ is the complement of $I(h)$. We show that some existing domain discrepancy methods~\cite{ben2007analysis} can be seen as special cases of this definition by using the following relaxed covariate shift assumption.
\begin{definition}\label{relaxed_CSA} ({\em Relaxed Covariate Shift Assumption})
    The source and target domains satisfy the relaxed covariate shift assumption if for any $h\in \mathcal{H}$,
    \begin{equation}
        \text{Pr}_{\mathcal{D}_S}[y ~|~ I(h)]=\text{Pr}_{\mathcal{D}_T}[y ~|~ I(h)] = \text{Pr}[y ~|~ I(h)]
    \end{equation}
\end{definition}
It would be equivalent to covariance shift assumption~\cite{shimodaira2000improving, johansson2019support} when $I(h)$ consists of only one example for all $h\in \mathcal{H}$ (see Lemma \ref{lemma:relaxed_CSA} for more details).
\begin{lemma}\label{joint_marginal}
With the relaxed covariate shift assumption, for any $h\in \mathcal{H}$, we have
\begin{equation*}
\begin{aligned}
d_{\mathcal{C}}(\mathcal{D}_S, \mathcal{D}_T) = \sup_{h\in \mathcal{H}} \Big|\Big(\text{Pr}_{\mathcal{D}_S}[I(h)] - \text{Pr}_{\mathcal{D}_T}[I(h)] \Big)\cdot \mathcal{S}_h + \text{Pr}_{\mathcal{D}_T}[y=1] - \text{Pr}_{\mathcal{D}_S}[y=1]\Big|
\end{aligned}
\end{equation*}
where $\mathcal{S}_h = \text{Pr}[y=1 | I(h)] - \text{Pr}[y=0 | I(h)]$.
\end{lemma}
\begin{remark}
 From Lemma \ref{joint_marginal}, we can see that in the special case where $\mathcal{S}_h$ is a constant for all $h\in\mathcal{H}$ and $\text{Pr}_{\mathcal{D}_T}[y=1] = \text{Pr}_{\mathcal{D}_S}[y=1]$, the proposed $\mathcal{C}$-divergence is reduced to the $\mathcal{A}$-distance~\cite{ben2007analysis} defined on the marginal distribution of features. More generally speaking, $\mathcal{C}$-divergence can be considered as a weighted version of the $\mathcal{A}$-distance where the hypothesis whose characteristic function has a larger class-separability (i.e., $|\mathcal{S}_h|$) receives a higher weight. Intuitively, compared to $\mathcal{A}$-distance, $\mathcal{C}$-divergence would pay less attention to class-inseparable regions in the input feature space, which provide meaningless information for task learning in domains.
\end{remark}

On the other hand, if the covariate shift assumption does not hold, previous work~\cite{wu2019domain,zhao2019learning} showed that the exact marginal distribution alignment might lead to undesirable performance in transfer learning. Here we provide the following lemma to illustrate a scenario where the same hypothesis might have significantly different error rates in the source and target domains.
\begin{lemma}\label{L: negative_transfer}
When $p_S(\mathbf{x})=p_T(\mathbf{x})$ and $\epsilon_S(h)=0$, if $\mathcal{L}(h(\mathbf{x}),y)=|h(\mathbf{x})-y|$, the $\mathcal{A}$-distance between the source and target domains would be 0. However,
\begin{equation*}
\begin{aligned}
    \epsilon_T(h) \geq \big| \text{Pr}_{\mathcal{D}_T}[y=1] - \text{Pr}_{\mathcal{D}_S}[y=1] \big|
\end{aligned}
\end{equation*}
\end{lemma}
\begin{remark}
 This lemma states that minimizing the source error and marginal domain discrepancy cannot guarantee the minimization of the target domain error due to the difference of marginal label distribution in these domains.
\end{remark}

With the proposed $\mathcal{C}$-divergence, we are able to avoid such scenarios. More specifically, the following theorem states that the target error is bounded in terms of $\mathcal{C}$-divergence and the expected source error.
\begin{theorem}\label{generalization}
Assume that loss function $\mathcal{L}$ is bounded, i.e., there exists a constant $M>0$ such that $0\leq\mathcal{L} \leq M$. For a hypothesis $h\in \mathcal{H}$, we have the following bound:
\begin{equation*}
    \epsilon_T(h) \leq \epsilon_S(h) + M\cdot d_{\mathcal{C}}(\mathcal{D}_S, \mathcal{D}_T)
\end{equation*}
\end{theorem}

\vspace{-2mm}
\subsection{Empirical Estimate of $\mathcal{C}$-divergence}
\vspace{-2mm}
In practice, it is difficult to calculate the proposed $\mathcal{C}$-divergence based on Eq.~(\ref{definition_divergence}) as it uses the true underlying distributions. Therefore, we propose the following empirical estimate of the $\mathcal{C}$-divergence between $\mathcal{D}_S$ and $\mathcal{D}_T$ as follows. Assuming that the hypothesis class $\mathcal{H}$ is symmetric (i.e., $1-h\in\mathcal{H}$ if $h\in \mathcal{H}$), the empirical $\mathcal{C}$-divergence is:
\begin{equation}\label{divergence_estimation}
% \begin{align}
    d_{\mathcal{C}}(\hat{\mathcal{D}}_S, \hat{\mathcal{D}}_T) = 1 - \min_{h\in \mathcal{H}} \Big| \frac{1}{m_S}\sum_{(\mathbf{x},y):h(\mathbf{x})\neq y}\mathbb{I}[(\mathbf{x},y)\in \hat{\mathcal{D}}_S] + \frac{1}{m_T}\sum_{(\mathbf{x},y):h(\mathbf{x})=y}\mathbb{I}[(\mathbf{x},y)\in \hat{\mathcal{D}}_T] \Big|
% \end{align}
\end{equation}
where $\hat{\mathcal{D}}_S$ and $\hat{\mathcal{D}}_T$ denote the source and target domains with finite samples, respectively. $\mathbb{I}[a]$ is the binary indicator function which is 1 if $a$ is true, 0 otherwise.

The following lemma provides the upper bound of the true $\mathcal{C}$-divergence using its empirical estimate.
\begin{lemma}\label{T: Rademacher_bound}
For any $\delta \in (0,1)$, with probability at least $1-\delta$ over $m_S$ labeled source examples $\mathcal{B}_S$ and $m_T$ labeled target examples $\mathcal{B}_T$, we have:
\begin{equation*}
    \begin{aligned}
    d_\mathcal{C}(\mathcal{D}_S, \mathcal{D}_T) \leq d_{\mathcal{C}}(\hat{\mathcal{D}}_S, \hat{\mathcal{D}}_T) + \Big(\hat{\Re}_{\mathcal{B}_S}(L_H) + \hat{\Re}_{\mathcal{B}_T}(L_H) \Big) + 3 \Bigg(\sqrt{\frac{\log{\frac{4}{\delta}}}{2m_S}} + \sqrt{\frac{\log{\frac{4}{\delta}}}{2m_T}} \Bigg)
    \end{aligned}
\end{equation*}
where $\hat{\Re}_{\mathcal{B}}(L_H) (\mathcal{B} \in \{\mathcal{B}_S, \mathcal{B}_T\})$ denote the Rademacher complexity~\cite{mansour2009domain} over $\mathcal{B}$ and $L_H = \{ (\mathbf{x},y) \rightarrow \mathbb{I}[h(\mathbf{x})=y]: h\in \mathcal{H} \}$ be a class of functions mapping $\mathcal{Z}=\mathcal{X}\times \mathcal{Y}$ to $\{0,1\}$.
\end{lemma}

\vspace{-2mm}
\section{Error Bounds with Empirical $\mathcal{C}$-divergence}\label{Guarantees}
\vspace{-2mm}
In this section, we provide the analysis of the error bounds using the empirical estimate of the proposed $\mathcal{C}$-divergence for various transfer learning settings.

\vspace{-2mm}
\subsection{Static Transfer Learning Scenario}
\vspace{-2mm}
When considering the transfer learning scenario with one source domain and one static target domain, we show that the expected target error is bounded in terms of the empirical estimate of the proposed $\mathcal{C}$-divergence and the empirical Rademacher complexity of function class $L_H$ as well as the number of labeled examples in both domains.
\begin{theorem}({Static Error Bound})\label{empirical_generalization_bound}
Assume the loss function $\mathcal{L}$ is bounded with $0\leq\mathcal{L} \leq M$. For $h\in \mathcal{H}$ and $\delta \in (0,1)$, with probability at least $1-\delta$ over $m_S$ examples $\mathcal{B}_S$ drawn from $\mathcal{D}_S$ and $m_T$ examples $\mathcal{B}_T$ drawn from $\mathcal{D}_T$, we have:
\begin{equation*} 
\small
    \begin{aligned}
    \epsilon_T(h) \leq \hat{\epsilon}_S(h) + M\Bigg(d_{\mathcal{C}}(\hat{\mathcal{D}}_S, \hat{\mathcal{D}}_T) + \hat{\Re}_{\mathcal{B}_S}(L_H) + \hat{\Re}_{\mathcal{B}_T}(L_H) + 3\sqrt{\frac{\log{\frac{8}{\delta}}}{2m_S}} + 3\sqrt{\frac{\log{\frac{8}{\delta}}}{2m_T}} + \sqrt{\frac{M^2\log{\frac{4}{\delta}}}{2m_S}} \Bigg)
    \end{aligned}
\end{equation*}
where $\hat{\epsilon}_S(h)$ denotes the empirical source error over finite data set $\mathcal{B}_S$.
\end{theorem}
\begin{remark}
  Compared with existing error bounds~\cite{ben2007analysis, mansour2009domain}, in Theorem \ref{empirical_generalization_bound}, the target error is bounded in terms of only data-dependent terms (e.g., empirical source error and $\mathcal{C}$-divergence), whereas existing error bounds are determined by the error terms that involve the intractable labeling function or optimal target hypothesis. In addition, we empirically show in Section \ref{Ex: empirical_bound} that our bound is much tighter than Rademacher complexity based bound in \cite{mansour2009domain}.
\end{remark}

\vspace{-2mm}
\subsection{Continuous Transfer Learning Scenario}
\vspace{-2mm}
Given a source domain and a time evolving target domain, continuous transfer learning aims to improve the target predictive function over $\mathcal{D}_{T_{t+1}}$ from the source domain and historical target domain. The error bound of continuous transfer learning is given by the following theorem.

\begin{theorem}({Continuous Error Bound})\label{continuous_transfer_learning_bound}
Assume the loss function $\mathcal{L}$ is bounded and $d_{\mathcal{C}}(\mathcal{D}_S, \mathcal{D}_{T_1}) \leq \Delta$, $d_{\mathcal{C}}(\mathcal{D}_{T_{i}}, \mathcal{D}_{T_{i+1}}) \leq \Delta$ for all $i=1,\cdots,n-1$ where $\Delta>0$. Then, for any $\delta>0$ and $h\in \mathcal{H}$, with probability at least $1-\delta$, the target domain error $\epsilon_{T_{t+1}}$ is bounded by
\begin{equation*}
    \begin{aligned}
        &\epsilon_{T_{t+1}}(h) \leq \frac{1}{t+1} \left( \hat{\epsilon}_S(h) + \sum_{i=1}^t \hat{\epsilon}_{T_i}(h) \right) + \frac{(t+2)M\Delta}{2} + \Tilde{\delta}
    \end{aligned}
\end{equation*}
where $\Tilde{\delta}=\frac{M}{t+1}\left( \sqrt{\frac{\log{\frac{2(t+1)}{\delta}}}{2m_S}} +\sum_{i=1}^t \sqrt{\frac{\log{\frac{2(t+1)}{\delta}}}{2m_{T_i}}} + \sqrt{\frac{2\log{\frac{2}{\delta}}}{m_{all}}} \right)$, $m_{all} = m_S + \sum_{i=1}^t m_{T_i}$ and $m_{T_i}$ is the number of labeled instances in $\mathcal{D}_{T_i}$.
\end{theorem}
This theorem states that the expected error of target domain at the $(t+1)^{\text{th}}$ time stamp is bounded by the historical estimated classification errors and the $\mathcal{C}$-divergence of the target domain between any consecutive time stamps as well as the $\mathcal{C}$-divergence between the source and initial target domains. 
% We would like to point out that if transferring the source domain $\mathcal{D}_S$ to target domain $\mathcal{D}_{T_{t+1}}$ directly, we have $d_{\mathcal{C}}(\mathcal{D}_S, \mathcal{D}_{T_{t+1}}) \leq (t+1) \cdot \Delta$ due to the triangle inequality of $\mathcal{C}$-divergence (see Lemma \ref{Triangle_property}), which might lead to the negative transfer when $t$ is large.

\vspace{-2mm}
\subsection{Negative Transfer}\label{negative_transfer_section}
\vspace{-2mm}
Informally, negative transfer is considered as the situation where transferring knowledge from the source domain has a negative impact on the target learner~\cite{wang2019characterizing}: $\epsilon_T(A(\mathcal{D}_S, \mathcal{D}_T)) > \epsilon_T(A(\emptyset, \mathcal{D}_T))$
% \begin{equation}\label{neg_transfer_error}
%     \epsilon_T(A(\mathcal{D}_S, \mathcal{D}_T)) > \epsilon_T(A(\emptyset, \mathcal{D}_T))
% \end{equation}
where $A$ is the learning algorithm. $\epsilon_T$ is the target error induced by this algorithm $A$. $\emptyset$ implies that it only considers the target data set for target learner.
Thus, in this paper, we define a \textbf{\em transfer signature} to measure the transferability from the source domain to the target domain as follows.
\begin{equation}
    TS(\mathcal{D}_T || \mathcal{D}_S)) = \inf_{A\in \mathcal{G}} \left( \epsilon_T\left(A(\mathcal{D}_S, \mathcal{D}_T)\right) - \epsilon_T\left(A(\emptyset, \mathcal{D}_T)\right) \right)
\end{equation}
where $\mathcal{G}$ is the set of all learning algorithms. We state that source domain knowledge is not transferable over target domain when $TS(\mathcal{D}_T || \mathcal{D}_S)) > 0$. Specially, since $A(\mathcal{D}_S, \mathcal{D}_T)$ learns an optimal classifier using both source and target data, we can define $\epsilon_T(A(\mathcal{D}_S, \mathcal{D}_T)) = \epsilon_T(h_{\alpha}^*)$
where $h_{\alpha}^*= \arg\min_{h\in \mathcal{H}(A)} \alpha \epsilon_T(h) + (1-\alpha) \epsilon_S(h)$ and $\mathcal{H}(A)$ is the hypothesis space induced by $A$. When we only consider the target domain with $\alpha=1$, $\epsilon_T(A(\emptyset, \mathcal{D}_T)) = \epsilon_T(h_{T}^*)$ where $h_{T}^* = \arg\min_{h\in \mathcal{H}(A)} \epsilon_T(h)$. Then we have the following theorem regarding the transfer signature.

\begin{theorem}({Transfer Signature Bound})\label{T: semi_bound}
Assume the loss function $\mathcal{L}$ is bounded with $0\leq\mathcal{L} \leq M$, we have
\begin{equation*}
\begin{aligned}
    \epsilon_T({h}_{\alpha}^*) &\leq \epsilon_T(h_T^*) + 2(1-\alpha) M d_{\mathcal{C}}(\mathcal{D}_S, \mathcal{D}_T)
\end{aligned}
\end{equation*}
Furthermore,
\begin{equation*}
    TS(\mathcal{D}_T || \mathcal{D}_S)) \leq  2(1-\alpha) M d_{\mathcal{C}}(\mathcal{D}_S, \mathcal{D}_T)
\end{equation*}
\end{theorem}
\begin{remark}
  Intuitively, we have the following observations: (1) Larger $\mathcal{C}$-divergence between domains is often associated with a higher transfer signature, which indicates that negative transfer can be characterized using the proposed $\mathcal{C}$-divergence; (2) Empirically, the larger amount of labeled target data could increase the value of $\alpha$, thus lead to the learned classifier relying more on the target data, which is consistent with the observation in \cite{wang2019characterizing}. One extreme case is that $\alpha=1$ implies we have adequate labeled target examples for standard supervised learning on the target domain without transferring knowledge from the source domain.
\end{remark}

\vspace{-2mm}
\section{Proposed Framework}\label{model}
\vspace{-2mm}
In this section, we present an adversarial Variational Auto-encoder (VAE) framework based on our proposed label-informed domain discrepancy.

\vspace{-2mm}
\subsection{Label-informed Adversarial VAE}\label{LI-VAE}
\vspace{-2mm}
We first consider the static transfer learning setting. In our framework (illustrated in Figure \ref{overview} in the appendices), we aim to learn a domain-invariant latent representation for both source and target domains such that the data distributions $p_S(\mathbf{x}, y)$ and $p_T(\mathbf{x}, y)$ could be well aligned in the latent feature space. Following the semi-supervised VAE~\cite{kingma2014semi}, we propose to learn the latent feature space by maximizing the following likelihood on both the source and target domains.
\begin{equation}
    \begin{aligned}
        \log{p_{\theta}(\mathbf{x},y)} &= KL \big(q_{\phi}(\mathbf{z}|\mathbf{x},y) || p_{\theta}(\mathbf{z}|\mathbf{x},y) \big) + \mathbb{E}_{q_{\phi}(\mathbf{z}|\mathbf{x},y)} [\log{p_{\theta}(\mathbf{x}, y, \mathbf{z})} - \log{q_{\phi}}(\mathbf{z}|\mathbf{x},y)]
    \end{aligned}
\end{equation}
where $\phi$ and $\theta$ are the learnable parameters in the encoder and decoder phases respectively. The evidence lower bound (ELBO), a lower bound on this log-likelihood, can be written as follows.
\begin{equation}
    \begin{aligned}
        &\mathcal{L}_{\theta, \phi}(\mathbf{x},y)
        = \mathbb{E}_{q_{\phi}(\mathbf{z}|\mathbf{x},y)} \left[\log{p_{\theta}(\mathbf{x},y|\mathbf{z})} \right] - KL \left(q_{\phi}(\mathbf{z}|\mathbf{x},y) || p(\mathbf{z}) \right)
    \end{aligned}
\end{equation}
where $\mathcal{L}_{\theta, \phi}(\mathbf{x},y) \leq \log{p_{\theta}(\mathbf{x},y)}$. Similarly, we have the following ELBO to maximize the log-likelihood of $p_{\theta}(\mathbf{x})$ when the label is not available:
\begin{equation}\label{eq:unlab_elbo}
    % \begin{aligned}
        \mathcal{U}_{\theta, \phi}(\mathbf{x},y) = \sum_{y} \left(q_{\phi}(y|\mathbf{x})\cdot \mathcal{L}_{\theta, \phi}(\mathbf{x},y) \right) - \mathbb{E}_{q_{\phi}(y|\mathbf{x})} \left[ \log q_{\phi}(y|\mathbf{x}) \right]
    % \end{aligned}
\end{equation}
where $p_{\theta}(\mathbf{x},y, \mathbf{z})=p_{\theta}(\mathbf{x}| y, \mathbf{z})p_{\theta}(y|\mathbf{z})p(\mathbf{z})$ with prior Gaussian distribution $p(\mathbf{z})=\mathcal{N}(\mathbf{0}, \mathbf{I})$.

In our framework, we propose to minimize the following objective function:
\begin{equation}\label{static_objective}
\small
        \mathcal{J}(S, T) = \underbrace{\sum_{i=1}^{m_S+m_T} \mathcal{L}_{clc} \left( y_i, q_{\phi}(\cdot|\mathbf{x}_i) \right)}_{\text{Classification error}} + \underbrace{d_{\mathcal{C}}\left(\hat{\mathcal{D}}_S, \hat{\mathcal{D}}_T\right)}_{\text{Estimated $\mathcal{C}$-divergence}} \nonumber - \lambda \underbrace{\left(\sum_{i=1}^{m_S+m_T} \mathcal{L}_{\theta, \phi}(\mathbf{x}_i,y_i) + \sum_{i=1}^{u_T} \mathcal{U}_{\theta, \phi}(\mathbf{x}_i,y_i) \right)}_{\text{ELBO on source and target examples}}
\end{equation}
where $u_T$ is the number of unlabeled training examples in the target domain, $q_{\phi}(\cdot)$ is the discriminative classifier formed by the distribution $q_{\phi}(y|\mathbf{x})$ in Eq.~(\ref{eq:unlab_elbo}) and $\lambda>0$ is a hyper-parameter. The first term $\mathcal{L}_{clc}$ is the cross-entropy loss function on labeled source and target examples. With the second term $d_{\mathcal{C}}(\hat{\mathcal{D}}_S, \hat{\mathcal{D}}_T)$, we aim to minimize the label-informed domain discrepancy in the latent feature space learned by maximizing the ELBO on the source and target domains.

We define $\Tilde{h}$ to be a two-dimensional characteristic function with $\Tilde{h}(\mathbf{x}, y)=1 \Leftrightarrow h(\mathbf{x})=y \Leftrightarrow \{h(\mathbf{x})=1, y=1\}\vee\{h(\mathbf{x})=0, y=0\}$ for $h\in \mathcal{H}$. Then the empirical $\mathcal{C}$-divergence in Eq.~(\ref{divergence_estimation}) can be rewritten as follows.
\begin{equation}\label{joint_divergence_estimation}
    d_{\mathcal{C}}(\hat{\mathcal{D}}_S, \hat{\mathcal{D}}_T) = 1 - \min_{\Tilde{h}} \Big| \frac{1}{m_S}\sum_{(\mathbf{x},y):\Tilde{h}(\mathbf{x}, y)=0}\mathbb{I}[(\mathbf{x},y)\in \hat{\mathcal{D}}_S] + \frac{1}{m_T}\sum_{(\mathbf{x},y):\Tilde{h}(\mathbf{x}, y)=1}\mathbb{I}[(\mathbf{x},y)\in \hat{\mathcal{D}}_T] \Big|
\end{equation}
Intuitively, by re-labeling each source example $(\mathbf{x},y)$ as 0 and target example as 1, the empirical $\mathcal{C}$-divergence can be derived by minimizing the domain classification accuracy with the hypothesis $\Tilde{h}$ because $\{(\mathbf{x},y):\Tilde{h}(\mathbf{x}, y)=0\}$ and $\{(\mathbf{x},y):\Tilde{h}(\mathbf{x}, y)=1\}$ indicate the correctly classified domain examples. Here, we adopted the domain-adversarial classifier~\cite{ganin2016domain} to calculate the empirical $\mathcal{C}$-divergence. Specifically, by mapping the example $(\mathbf{x},y)$ into a label-informed latent features $\mathbf{z}$, a domain classifier is trained to identify whether an example comes from the source or target domain in the latent feature space.

However, there are two limitations when applying our framework to real transfer learning scenarios: (1) It is difficult to estimate the $\mathcal{C}$-divergence with little labeled target examples when $m_S \gg m_T$; (2) When learning the latent feature $\mathbf{z}$, combining the data $\mathbf{x}$ (e.g., one image) and class-label $y$ directly might lead to over-emphasizing the data itself due to its high dimensionality compared to $y$. To mitigate these problems, we propose the following {\it Pseudo-label Inference}, i.e., we infer the pseudo labels of unlabeled examples using the classifier $q_{\phi}(y|\mathbf{x})$ for each training epoch. Using labeled source and target examples as well as unlabeled target examples with inferred pseudo labels, the $\mathcal{C}$-divergence could be estimated in a balanced setting. Furthermore, to enforce the compatibility between features $\mathbf{x}$ and label $y$, we adopt a pre-encoder step to learn a dense representation for the input $\mathbf{x}$, and then learn the label-informed latent feature $\mathbf{z}$.

\vspace{-2mm}
\subsection{Learning Time Evolving Target Domain}
\vspace{-2mm}
For continuous transfer learning, we leverage both the source domain and historical target domain data to learn the predictive function for the current time stamp. In other words, the objective function for the target domain $\mathcal{D}_{T_{t+1}}$ can be defined as $\mathcal{J}(T_{t}, T_{t+1})$. In this paper, we present an iterative optimization method to learn the optimal predictive function on target domain $\mathcal{D}_{T_{t+1}}$.

\begin{wrapfigure}{r}{0.6\textwidth}
    \begin{minipage}{0.6\textwidth}
    \vspace{-6mm}
      \setlength{\textfloatsep}{4pt}
\begin{algorithm}[H]
   \caption{Continuous Transfer Learning (\model)}
   \label{algorithm:CTL}
\begin{algorithmic}[1]
    \STATE \textbf{Input:} Source domain $\mathcal{D}_S$, evolving target domain $\{\mathcal{D}_{T_i}\}_{i=1}^t$, hyper-parameter $\lambda$.
    \STATE \textbf{Output:} Predictive function on $(t+1)^{\text{th}}$ target domain.
    \FOR{$i$ {\bfseries in} $[0, 1,\cdots,t]$}
    \IF{$i=0$}
        \STATE Minimizing the $\mathcal{J}(S, T_1)$ using Eq. (\ref{static_objective})
    \ELSE
        \STATE Generate pseudo-labels on unlabeled data in $\mathcal{D}_{T_i}$ using learned $q_{\phi}(y|\mathbf{x})$
        \STATE Minimizing the $\mathcal{J}(T_i, T_{i+1})$ using Eq. (\ref{static_objective})
    \ENDIF
    \STATE Output predictive function $q_{\phi}(y|\mathbf{x})$ on $\mathcal{D}_i$
    \ENDFOR
\end{algorithmic}
\end{algorithm}
    \end{minipage}
\end{wrapfigure}
As illustrated in Algorithm \ref{algorithm:CTL}, we first learn the predictive function on the first target domain using the knowledge from source domain by minimizing the objective function $\mathcal{J}(S, T_1)$ (shown in Step 5). Then it predicts the labels for the target domain at time stamp 1, which will be used to learn the predictive function for the target domain at time stamp 2 (shown in Step 7-8). Repeat this procedure until the predictive function on $(t+1)^{\text{th}}$ target domain is optimized. This allows us to optimize the predictive function at any time stamp using the knowledge from source domain and historical target domain.

The transferability from $S$ to $T_{t+1}$ could be identified using the empirical $\mathcal{C}$-divergence between source and target domains. It can be seen that (1) when the examples are indistinguishable for domain classifier, the empirical $\mathcal{C}$-divergence in Eq.~(\ref{joint_divergence_estimation}) would be small, which indicates the high transferability between domains; (2) on the other hand, when they are highly domain-separable, the empirical $\mathcal{C}$-divergence would be large, which might significantly enlarge the transfer signature based on Theorem \ref{T: semi_bound}, thus leading to negative transfer between source domain $S$ and target domain $T_{t+1}$.

\vspace{-2mm}
\section{Experimental Results}\label{Experiments}
\vspace{-2mm}
\subsection{Experiment Setup}
\vspace{-2mm}
\textbf{Synthetic Data:} We generate a simple synthetic data set to validate the label-informed distribution alignment. For this data set, each domain has 1000 positive examples and 1000 negative examples randomly generated from Gaussian distributions $\mathcal{N}([1.5\cos{\theta}, 1.5\sin{\theta}]^T, 0.5 \cdot \mathbf{I}_{2\times 2})$ and $\mathcal{N}([1.5\cos{(-\theta)}, 1.5\sin{(-\theta)}]^T, 0.5 \cdot \mathbf{I}_{2\times 2})$, respectively. We let $\theta=0$ for source domain (denoted as $S1$), and $\theta= \frac{i\cdot \pi}{n} (i=1,\cdots,n)$ for the time evolving target domain with $n=8$ time stamps (denoted as T1, $\cdots$, T8).

\textbf{Real Data:} We used the publicly available data sets: MNIST and SVHN. For a pair of SVHN and MNIST, we take the SVHN as source domain, and construct the time evolving target domain using MNIST. More specifically, we add the adversarial noise to the original MNIST images where the adversarial noise is learned by Fast Gradient Sign Method (FGSM)~\cite{goodfellow2014explaining}. The perturbation magnitude $\omega$ varies from 0.0 to 0.50 with an interval of 0.05, and the generated target domain at different time stamps are denoted as T1, $\cdots$, T11 respectively. For each time stamp in target domain, the number of labeled target training examples is set as 100.

\textbf{Baselines:} The baseline methods in our experiments are as follows. (1) SourceOnly: training with only source data; (2) {TargetOnly}: training with only target data assuming all the training target data are labeled. (3) TargetERM: empirical risk minimization (ERM) on only target domain; (4) CORAL~\cite{sun2016return}, DANN~\cite{ganin2016domain}, ADDA~\cite{tzeng2017adversarial}, WDGRL~\cite{shen2017wasserstein} and DIFA~\cite{volpi2018adversarial}: training with feature distribution alignment. (5) \model{}: training with label-informed distribution alignment on the evolving target domain. (6) \model{}\_p: a one-time transfer learning variant of \model{} that directly transfers from source domain to current time stamp in target domain. We fix $\lambda=0.05$ for all experiments, and all the methods use the same neural network architecture for feature extraction (or pre-encoding). Please see Section~\ref{appendix_ex} for more experiments in the appendices.

\vspace{-2mm}
\subsection{Evaluation of $\mathcal{C}$-divergence}\label{Ex: empirical_evaluation}
\vspace{-2mm}

\begin{wrapfigure}{r}{0.5\textwidth}
  \vspace{-7mm}
  \begin{center}
    \includegraphics[width=0.5\textwidth]{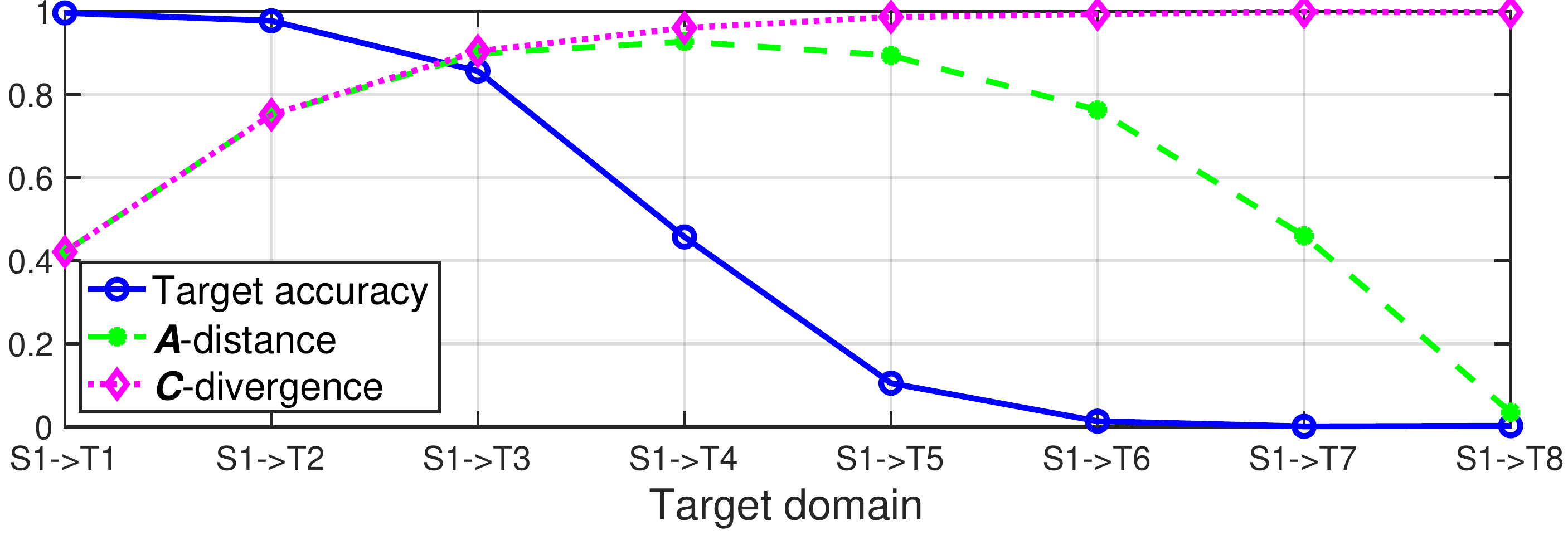}
  \end{center}
  \vspace{-5mm}
  \caption{Comparison of domain discrepancy and target accuracy}
  \label{fig:syn}
  \vspace{-2mm}
\end{wrapfigure}
 We compare the proposed $\mathcal{C}$-divergence with conventional domain discrepancy measure $\mathcal{A}$-distance~\cite{ben2007analysis} on a synthetic data set with an evolving target domain. We assumed the hypothesis space $\mathcal{H}$ to be consisting of the linear classifiers in feature space. Figure \ref{fig:syn} shows the domain discrepancy and target classification accuracy for each pair of source and target domains. We have the following observations. (1) The classification accuracy on the target domain significantly decreases from target domain T1 to target domain T8. One explanation is that the joint distribution $p(x,y)$ on the time evolving target domain is gradually shifted. (2) The $\mathcal{A}$-distance increases from S1$\rightarrow$T1 to S1$\rightarrow$T4, and then decreases from S1$\rightarrow$T4 to S1$\rightarrow$T8. That is because it only estimates the difference of the marginal data distribution $p(x)$ between source and target domains. (3) The $\mathcal{C}$-divergence keeps increasing from S1$\rightarrow$T1 to S1$\rightarrow$T8, which indicates the decreasing of task relatedness between the source domain and the target domain. Therefore, it provides an insight into avoiding the negative transfer by minimizing the $\mathcal{C}$-divergence between source and target domains in the latent feature space.

\vspace{-2mm}
\subsection{Evaluation of Error Bound}\label{Ex: empirical_bound}
\vspace{-2mm}

\begin{wrapfigure}{r}{0.5\textwidth}
  \vspace{-7mm}
  \begin{center}
    \includegraphics[width=0.5\textwidth]{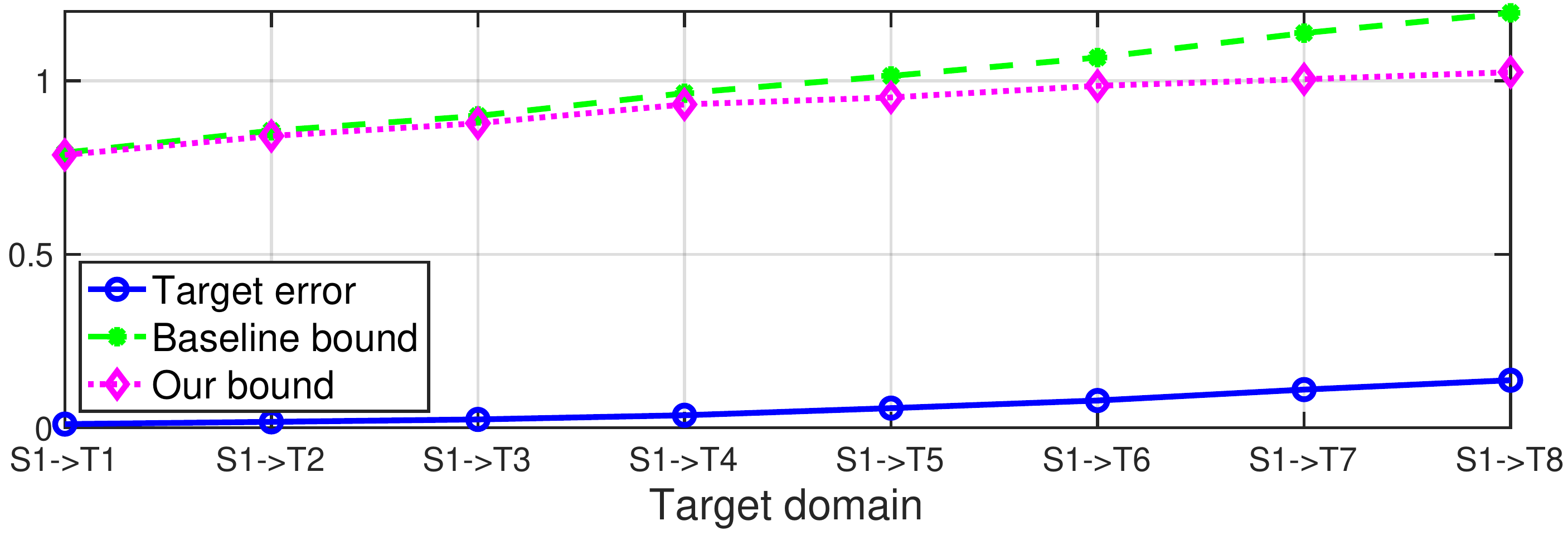}
  \end{center}
  \vspace{-5mm}
  \caption{Comparison of error bounds}
  \label{fig:bound}
  \vspace{-2mm}
\end{wrapfigure}
We empirically evaluate our derived error bound in Theorem \ref{empirical_generalization_bound} compared to the Rademacher complexity based error bound in \cite{mansour2009domain} (shown in Theorem \ref{baseline_bound} for being self-contained). In our experiments, we used the 0-1 loss function as $\mathcal{L}$ and assumed the hypothesis space $\mathcal{H}$ to be consisting of the linear classifiers in feature space. Figure \ref{fig:bound} shows the estimated error bounds and target error with the time evolving target domain (i.e., S1$\rightarrow$T1, $\cdots$, S1$\rightarrow$T8 in a new synthetic data set with a slower time evolving target domain to ensure that the baseline bound is meaningful most of the time) where we choose $h=h^*_S$. It demonstrates that our derived error bound is much tighter than the baseline. We would like to point out that when transferring source domain S1 to target domain T8, our error bound is largely determined by the $\mathcal{C}$-divergence, whereas the baseline is determined by the difference between the optimal source and target hypothesizes. Furthermore, given any hypothesis $h\in \mathcal{H}$, the baseline might not be able to estimate the error bound when the optimal hypothesis is not available.

\begin{table*}[t]
\vspace{-5mm}
\small
\centering
\setlength\tabcolsep{2.9pt}
\caption{Transfer learning accuracy from SVHN (source) to continuously evolving MNIST (target)}\label{svhn_mnist}
\begin{tabular}{|l|c|c|c|c|c|c|c|c|c|c|c|}
\hline
Target Domain & T1              & T2              & T3              & T4              & T5              & T6              & T7              & T8              & T9              & T10             & T11             \\ \hline\hline
SourceOnly    & 0.6998          & 0.6738          & 0.6336          & 0.5692          & 0.4747          & 0.4110          & 0.3087          & 0.2220          & 0.1481          & 0.0828          & 0.0764          \\ \hline
TargetOnly    & 0.9887          & 0.9918          & 0.9974          & 0.9976          & 0.9975          & 0.9976          & 0.9975          & 0.9971          & 0.9969          & 0.9972          & 0.9974          \\ \hline\hline
TargetERM & 0.7451 & 0.6997 & 0.6618 & 0.6314 & 0.6368 & 0.6359 & 0.6695 & 0.7133 & 0.7214 & 0.7450 & 0.7512 \\ \hline
CORAL~\cite{sun2016return}         & 0.8349          & 0.8410          & 0.7633          & 0.7063          & 0.6496          & 0.5900          & 0.5031          & 0.5101          & 0.4337          & 0.4156          & 0.4502          \\ \hline
DANN~\cite{ganin2016domain}          & 0.8666          & 0.8356          & 0.8018          & 0.7529          & 0.7309          & 0.6641          & 0.6614          & 0.5618          & 0.5204          & 0.5082          & 0.4594          \\ \hline
ADDA~\cite{tzeng2017adversarial}          & 0.8667          & 0.8487          & 0.7982          & 0.7187          & 0.6804          & 0.5397          & 0.4366          & 0.3473          & 0.2636          & 0.1659          & 0.1259          \\ \hline
WDGRL~\cite{shen2017wasserstein}         & 0.8990          & 0.8602          & 0.8247          & 0.8222          & 0.7452          & 0.6877          & 0.6481          & 0.5896          & 0.5145          & 0.4952          & 0.5196          \\ \hline
DIFA~\cite{volpi2018adversarial}          & 0.9164          & 0.8993          & 0.8713          & 0.8273          & 0.7935          & 0.6661          & 0.5956          & 0.4381          & 0.3479          & 0.2448          & 0.1332          \\ \hline\hline
TransLATE\_p  & 0.9621          & 0.9213          & 0.8977          & 0.8901          & 0.8274          & 0.8145          & 0.7360          & 0.7256          & 0.6199          & 0.6774          & 0.7009          \\ \hline
TransLATE     & \textbf{0.9621} & \textbf{0.9575} & \textbf{0.9520} & \textbf{0.9480} & \textbf{0.9488} & \textbf{0.9430} & \textbf{0.9389} & \textbf{0.9420} & \textbf{0.9453} & \textbf{0.9531} & \textbf{0.9663} \\ \hline
\end{tabular}
\vspace{-5mm}
\end{table*}

\vspace{-2mm}
\subsection{Evaluation of Continuous Transfer Learning}\label{experiments_CTL}
\vspace{-2mm}
Table \ref{svhn_mnist} provides the continuous transfer learning results on digital data sets where the classification accuracy on target domain is reported (the best results are indicated in bold). It is observed that (1) the classification accuracy using SourceOnly algorithm significantly decreased on the evolving target domain due to the shift of joint data distribution $p(\mathbf{x},y)$ on target domain; (2) TargetOnly achieves the satisfactory results, which indicates that the generated evolving target domain keeps highly class-separable; (3) transfer learning algorithms outperform TargetERM for T1-T5, whereas negative transfer might happen for T6-T11 when data distribution between source and current target tasks are largely shifted; (4) \model{} significantly outperformed \model{}\_p as well as other baseline algorithms on target domain because the historical target domain knowledge allows to smoothly re-align the target distribution when the change of target domain data distribution in consecutive time stamps is small.

\vspace{-2mm}
\section{Related Work}\label{related_work}
\vspace{-2mm}
{\bf Transfer Learning:}
Transfer learning~\cite{pan2009survey,ying2018transfer,jang2019learning} improves the performance of a learning algorithm on the target domain by using the knowledge from the source domain. Theoretically, it is proven that the target error is bounded by the source error, domain discrepancy and labeling difference between the source and target domains~\cite{ben2007analysis,ben2010theory, mansour2009domain, zhang2019bridging}, followed by a lot of practical transfer learning algorithms~\cite{ganin2016domain,shen2017wasserstein,long2017deep,long2018conditional,saito2018maximum,chen2019transferability} with covariate shift assumption. However, it is observed that this assumption does not always hold in real-world scenarios~\cite{rosenstein2005transfer, zhao2019learning, johansson2019support,wang2019characterizing}. In this paper, we proposed to study the transferability between a source domain and a time evolving target domain via label-informed $\mathcal{C}$-divergence. Besides, compared to domain divergences in~\cite{mohri2012new,zhang2012generalization}, our proposed $\mathcal{C}$-divergence is derived from the perspective of measurable set matching, thus shedding light on the empirical estimate of label-informed domain discrepancy from finite samples in practice.

{\bf Continual Learning:} 
Continual lifelong learning~\cite{ruvolo2013ella,rannen2017encoder,parisi2019continual,rusu2016progressive,hoffman2014continuous, bobu2018adapting} involves the sequential learning tasks with the goal of learning a predictive function on the new task using knowledge from historical tasks. Most of them focused on mitigating catastrophic forgetting when learning new tasks from only one domain, whereas our work studied the transferability between a source domain and a time evolving target domain. Besides, little work has been devoting to characterizing the potential negative transfer induced by a source domain and a time evolving target domain.

\vspace{-2mm}
\section{Conclusion}\label{conclusion}
\vspace{-2mm}
In this paper, we study a novel continuous transfer learning setting with a time evolving target domain. We start by proposing a label-informed $\mathcal{C}$-divergence to measure the domain discrepancy induced by the shift of joint data distributions. Then we provide the error bounds of continuous transfer learning in terms of the empirical $\mathcal{C}$-divergence, and characterize the negative transfer which might appear due to the cumulative change of the target domain. Following the theoretical analysis, we propose a generic adversarial Variational Auto-encoder framework named \model{} for continuous transfer learning. Extensive experiments on both synthetic and real data sets demonstrate the effectiveness of our \model{} framework.

\clearpage
\section*{Broader Impact}

This work focused on providing a theoretical analysis on the continuous transfer learning problem, followed by a practical continuous transfer learning framework. Generally speaking, the communities working on inferring the object's behavior from historical data might benefit from our paper. However, one common ethical concern for leveraging historical data is privacy and security. The malicious manipulation on historical data might provide deceitful and misleading information on understanding the object's behavior. In this paper, we characterize the negative transfer using our proposed $\mathcal{C}$-divergence, thereby leading to identify whether such malicious manipulations happen or not.

% Authors are required to include a statement of the broader impact of their work, including its ethical aspects and future societal consequences. 
% Authors should discuss both positive and negative outcomes, if any. For instance, authors should discuss a) 
% who may benefit from this research, b) who may be put at disadvantage from this research, c) what are the consequences of failure of the system, and d) whether the task/method leverages
% biases in the data. If authors believe this is not applicable to them, authors can simply state this.

% Use unnumbered first level headings for this section, which should go at the end of the paper. {\bf Note that this section does not count towards the eight pages of content that are allowed.}

\bibliographystyle{plain}
\bibliography{reference}

\clearpage
\appendix
\section{Appendices}
To better reproduce the experimental results, we provide additional details here.

\subsection{Notation}
The main notation used in this paper is summarized in Table \ref{tab: notations}.

\begin{table}[htp]
\centering
\caption{Notation} \label{tab: notations}
\begin{tabular}{|c|l|}
\hline
\multicolumn{1}{|c|}{Notation} & \multicolumn{1}{c|}{Definition}         \\ \hline\hline
$\mathcal{X}, \mathcal{Y}$, $\mathcal{Z}$ & Input space, class space, latent feature space \\ \hline
$\mathcal{D}_S$, $\{\mathcal{D}_{T_t}\}_{t=1}^n$ & Source domain and evolving target domain \\ \hline
$\epsilon_S, \hat{\epsilon}_S$ & Expected and estimated source error \\ \hline
$\epsilon_T, \hat{\epsilon}_T$ & Expected and estimated target error \\ \hline
$p_S, p_T$ & Probability density functions (pdf) \\\hline
$\text{Pr}_{\mathcal{D}_S}, \text{Pr}_{\mathcal{D}_T}$ & Probability mass functions (pmf) \\ \hline
$m_S$, $m_T$ & Number of labeled source and target samples \\ \hline
\end{tabular}
\end{table}

\subsection{Theoretical Analysis}
We first introduce some useful existing lemmas and theorems, followed by the details regarding the proof for lemmas and theorems involved in this paper.

\subsubsection{Existing Definitions, Lemmas and Theorems}
For being self-contained in this paper, we provide some exiting lemmas and theorems as follows.

\begin{definition}\label{def:rademacher_complexity} ({\em Rademacher Complexity}~\cite{mansour2009domain})
    Given a set of real-valued functions $\mathcal{F}$ over $\mathcal{X}$ and an example $\mathcal{B}=\{ \mathbf{x}_1,\cdots, \mathbf{x}_m\} \in \mathcal{X}^m$, the empirical Rademacher complexity of $\mathcal{F}$ is defined as follow:
    \begin{equation*}
        \hat{\Re}_{\mathcal{B}}(\mathcal{F}) = \frac{2}{m} \mathbb{E}_{\sigma}\Big[ \sup_{f\in \mathcal{F}} \big| \sum_{i=1}^m \sigma_i f(\mathbf{x}_i) \big| \Big| \mathcal{B}=\{ \mathbf{x}_1,\cdots, \mathbf{x}_m\} \Big]
    \end{equation*}
\end{definition}
where $\sigma=(\sigma_1, \cdots, \sigma_m)$ with each $\sigma_i$ sampling from two values $\{-1, +1\}$ according to an independent and uniform distribution.

\begin{lemma}\label{McDiarmid} ({McDiarmid's inequality})
Let $X_1, \cdots, X_m$ be independently random variables taking values in the set $\mathcal{X}$ and $f: \mathcal{X}^m \rightarrow \mathbb{R}$ be a function over $X_1, \cdots, X_m$ that satisfies $\forall i, \forall x_1, \cdots, x_m, x'_i \in \mathcal{X}$,
\begin{equation*}
    \left| f(x_1,\cdots,x_i,\cdots,x_m) - f(x_1,\cdots,x'_i,\cdots,x_m) \right| \leq c_i
\end{equation*}
Then, for any $\epsilon > 0$,
\begin{equation*}
    \text{Pr}\left[ f-\mathbb{E}[f] \geq \epsilon \right] \leq \exp{\left( \frac{-2\epsilon^2}{\sum_{i=1}^m c_i^2} \right)}
\end{equation*}
\end{lemma}

\begin{lemma} ({Hoeffding's inequality})
If $X_1, \cdots, X_m$ are independently random variables with $a_i \leq X_i \leq b_i$, then for any $\epsilon > 0$,
\begin{equation*}
    \text{Pr}[|\Bar{X} - \mathbb{E}[\Bar{X}]|\geq \epsilon] \leq 2\exp{\left(\frac{-2m^2\epsilon^2}{\sum_{i=1}^m(b_i-a_i)^2} \right)}
\end{equation*}
where $\Bar{X} = (X_1+\cdots+X_m)/m$ and $\mathbb{E}[\Bar{X}]$ is the expectation over $\Bar{X}$.
\end{lemma}

We restate the conventional error bound based on Rademacher complexity (see Theorem 8 in \cite{mansour2009domain}) as follows.
\begin{theorem}({Error Bound in \cite{mansour2009domain}})\label{baseline_bound}
Assume that the loss function $\mathcal{L}$ is symmetric and obeys the triangle inequality. Then, for any hypothesis $h\in \mathcal{H}$, the following holds
\begin{equation*}
    \epsilon_T(h) \leq \epsilon_T(h^*_T) + \mathbb{E}_{\mathbf{x} \sim p_S(\mathbf{x})} \big[\mathcal{L}(h(\mathbf{x}), h^*_S(\mathbf{x})) \big] + \mathbb{E}_{\mathbf{x} \sim p_S(\mathbf{x})} \big[\mathcal{L}(h^*_T(\mathbf{x}), h^*_S(\mathbf{x})) \big] + d_{\mathcal{L}}(\mathcal{D}_S, \mathcal{D}_T)
\end{equation*}
where $d_{\mathcal{L}}(\mathcal{D}_S, \mathcal{D}_T) = \max_{h,h'\in\mathcal{H}} \left| \mathbb{E}_{\mathbf{x} \sim p_S(\mathbf{x})} \big[\mathcal{L}(h(\mathbf{x}), h'(\mathbf{x})) \big] - \mathbb{E}_{\mathbf{x} \sim p_T(\mathbf{x})} \big[\mathcal{L}(h(\mathbf{x}), h'(\mathbf{x})) \big] \right|$, and $h_S^*, h_T^*$ denote the optimal hypothesises of $\epsilon_S(h)$ and $\epsilon_T(h)$, respectively.
\end{theorem}

\subsubsection{Our Results}
Then we provide the theoretical analysis and proof regarding our lemmas and theorems used in this paper as follows.

\begin{lemma}\label{estimation_error}
Assume that loss function $\mathcal{L}$ is bounded, i.e., there exists $M>0$ such that $0\leq\mathcal{L} \leq M$. For $h\in \mathcal{H}$ and $\delta \in (0,1)$, with probability at least $1-\delta$ over $m$ samples $\mathcal{B}_S$ drawn from $\mathcal{D}_S$, we have:
\begin{equation*}
    \text{Pr}[|\hat{\epsilon}_S(h) - \epsilon_S(h)|\geq \epsilon] \leq 2\exp{(-2m\epsilon^2/M^2)}
\end{equation*}
\end{lemma}
\begin{proof}
It simply follows the Hoeffding's equality considering $0\leq\mathcal{L}(h(\mathbf{x}), y) \leq M$ for each sample.
\end{proof}

\begin{lemma}\label{lemma:relaxed_CSA} ({Property of Relaxed Covariate Shift Assumption}) If the covariate shift assumption between source and target domains holds, and source and target examples follow the IID assumption w.r.t. $p_S(\mathbf{x},y)$ and $p_T(\mathbf{x},y)$ respectively, then the relaxed covariate shift assumption holds. Furthermore, it would be equivalent to covariance shift assumption when $I(h)$ consists of only one example for all $h\in \mathcal{H}$.
\end{lemma}
\begin{proof}
For either source or target domain, if its examples follow the IID assumption, then we have
\begin{equation*}
    \begin{aligned}
        \text{Pr}(y | I(h)) \text{Pr}(I(h)) &= \text{Pr}(y, I(h)) =  \text{Pr}(y, \mathbf{x}_1)\cdots \text{Pr}(y, \mathbf{x}_n) \\
        &= \text{Pr}(y | \mathbf{x}_1)\text{Pr}(\mathbf{x}_1)\cdots \text{Pr}(y | \mathbf{x}_n)\text{Pr}(\mathbf{x}_n) \\
        &= \text{Pr}(y | \mathbf{x}_1)\cdots \text{Pr}(y | \mathbf{x}_n) \text{Pr}(I(h))
    \end{aligned}
\end{equation*}
where we denote $\mathbf{x}_1,…,\mathbf{x}_n$ are the data points in the set $I(h)$. Then if covariate shift assumption holds, i.e., $\text{Pr}_S(y | \mathbf{x}_i)=\text{Pr}_T(y | \mathbf{x}_i)$ for all examples $\mathbf{x}_1,…,\mathbf{x}_n$, we have $\text{Pr}_S(y | I(h))= \text{Pr}_T(y | I(h))$ as shown in the relaxed covariance shift assumption (see Definition \ref{relaxed_CSA}). It is easy to show that when $I(h)$ consists of only one example for all $h\in \mathcal{H}$, it is equivalent to covariance shift assumption.
\end{proof}

\begin{lemma}\label{Triangle_property} ({Triangle Inequality of $\mathcal{C}$-divergence})
Given domains $\mathcal{D}_1$, $\mathcal{D}_2$ and $\mathcal{D}_3$, the $\mathcal{C}$-divergence satisfies the following triangle property:
\begin{equation}
    d_{\mathcal{C}}(\mathcal{D}_1, \mathcal{D}_{2}) \leq d_{\mathcal{C}}(\mathcal{D}_1, \mathcal{D}_{3}) + d_{\mathcal{C}}(\mathcal{D}_2, \mathcal{D}_{3})
\end{equation}
\end{lemma}
\begin{proof}
Following the definition of $\mathcal{C}$-divergence in Eq. (\ref{definition_divergence}), it is easy to show the $\mathcal{C}$-divergence is symmetric with respect to its two arguments. Then we have
\begin{equation*}
    \begin{aligned}
        d_{\mathcal{C}}(\mathcal{D}_1, \mathcal{D}_2) &= \sup_{h\in \mathcal{H}} \Big|\text{Pr}_{\mathcal{D}_1}[\{I(h), y=1\}\cup \{\overline{I(h)}, y=0\}] - \text{Pr}_{\mathcal{D}_2}[\{I(h), y=1\}\cup \{\overline{I(h)}, y=0\}] \Big| \\
        &= \sup_{h\in \mathcal{H}} \Big|\text{Pr}_{\mathcal{D}_1}[\{I(h), y=1\}\cup \{\overline{I(h)}, y=0\}] - \text{Pr}_{\mathcal{D}_3}[\{I(h), y=1\}\cup \{\overline{I(h)}, y=0\}] \\
        &\quad + \text{Pr}_{\mathcal{D}_3}[\{I(h), y=1\}\cup \{\overline{I(h)}, y=0\}] - \text{Pr}_{\mathcal{D}_2}[\{I(h), y=1\}\cup \{\overline{I(h)}, y=0\}] \Big| \\
        &\leq \sup_{h\in \mathcal{H}} \Big|\text{Pr}_{\mathcal{D}_1}[\{I(h), y=1\}\cup \{\overline{I(h)}, y=0\}] - \text{Pr}_{\mathcal{D}_3}[\{I(h), y=1\}\cup \{\overline{I(h)}, y=0\}] \Big| \\
        &\quad + \sup_{h\in \mathcal{H}} \Big| \text{Pr}_{\mathcal{D}_3}[\{I(h), y=1\}\cup \{\overline{I(h)}, y=0\}] - \text{Pr}_{\mathcal{D}_2}[\{I(h), y=1\}\cup \{\overline{I(h)}, y=0\}] \Big| \\
        &= d_{\mathcal{C}}(\mathcal{D}_1, \mathcal{D}_{3}) + d_{\mathcal{C}}(\mathcal{D}_2, \mathcal{D}_{3})
    \end{aligned}
\end{equation*}
which completes the proof.
\end{proof}

{\bf Proof of Lemma~\ref{joint_marginal}}.
Lemma~\ref{joint_marginal} states that with relaxed covariate shift assumption, for any $h\in \mathcal{H}$, we have
\begin{equation*}
\begin{aligned}
d_{\mathcal{C}}(\mathcal{D}_S, \mathcal{D}_T) = \sup_{h\in \mathcal{H}} \Big|\Big(\text{Pr}_{\mathcal{D}_S}[I(h)] - \text{Pr}_{\mathcal{D}_T}[I(h)] \Big)\cdot \mathcal{S}_h + \text{Pr}_{\mathcal{D}_T}[y=1] - \text{Pr}_{\mathcal{D}_S}[y=1]\Big|
\end{aligned}
\end{equation*}
where
\begin{equation*}
    \mathcal{S}_h = \text{Pr}[y=1 | I(h)] - \text{Pr}[y=0 | I(h)]
\end{equation*}
\begin{proof}
For any $h\in \mathcal{H}$, we have
\begin{equation*}
    \begin{aligned}
        &\text{Pr}_{\mathcal{D}_S}[\{I(h), y=1\}\cup \{\overline{I(h)}, y=0\}] - \text{Pr}_{\mathcal{D}_T}[\{I(h), y=1\}\cup \{\overline{I(h)}, y=0\}] \\
        &=\text{Pr}_{\mathcal{D}_S}[I(h), y=1] + \text{Pr}_{\mathcal{D}_S}[y=0] -\text{Pr}_{\mathcal{D}_S}[I(h), y=0] \\
        &\quad - \text{Pr}_{\mathcal{D}_T}[I(h), y=1] - \text{Pr}_{\mathcal{D}_T}[y=0] + \text{Pr}_{\mathcal{D}_T}[I(h), y=0] \\
        &= 2\text{Pr}_{\mathcal{D}_S}[I(h), y=1] + 1 - \text{Pr}_{\mathcal{D}_S}[y=1] - \text{Pr}_{\mathcal{D}_S}[I(h)] \\
        &\quad - 2\text{Pr}_{\mathcal{D}_T}[I(h), y=1] - (1 - \text{Pr}_{\mathcal{D}_T}[y=1]) + \text{Pr}_{\mathcal{D}_T}[I(h)] \\
        &= \Big(\text{Pr}_{\mathcal{D}_S}[I(h)] - \text{Pr}_{\mathcal{D}_T}[I(h)] \Big)\Big(2\text{Pr}_{\mathcal{D}_S}[y=1 | I(h)] - 1 \Big) + \Big(\text{Pr}_{\mathcal{D}_T}[y=1] - \text{Pr}_{\mathcal{D}_S}[y=1]\Big) \\
        &\quad + \text{Pr}_{\mathcal{D}_T}[I(h)]\Big( \text{Pr}_{\mathcal{D}_S}[y=1 | I(h)] - \text{Pr}_{\mathcal{D}_T}[y=1 | I(h)] \Big)
    \end{aligned}
\end{equation*}
With the relaxed covariate shift assumption $\text{Pr}_{\mathcal{D}_S}[y ~|~ I(h)]=\text{Pr}_{\mathcal{D}_T}[y ~|~ I(h)]= \text{Pr}[y ~|~ I(h)]$, we have
\begin{equation*}
    \begin{aligned}
    &d_{\mathcal{C}}(\mathcal{D}_S, \mathcal{D}_T) = \sup_{h\in \mathcal{H}} \Big|\text{Pr}_{\mathcal{D}_S}[\{I(h), y=1\}\cup \{\overline{I(h)}, y=0\}] - \text{Pr}_{\mathcal{D}_T}[\{I(h), y=1\}\cup \{\overline{I(h)}, y=0\}] \Big| \\
        &= \sup_{h\in \mathcal{H}} \Big| \Big(\text{Pr}_{\mathcal{D}_S}[I(h)] - \text{Pr}_{\mathcal{D}_T}[I(h)] \Big)\Big(2\text{Pr}[y=1 | I(h)] - 1 \Big) + \Big(\text{Pr}_{\mathcal{D}_T}[y=1] - \text{Pr}_{\mathcal{D}_S}[y=1]\Big) \Big| \\
        &= \sup_{h\in \mathcal{H}} \Big|\Big(\text{Pr}_{\mathcal{D}_S}[I(h)] - \text{Pr}_{\mathcal{D}_T}[I(h)] \Big)\cdot \mathcal{S}_h + \text{Pr}_{\mathcal{D}_T}[y=1] - \text{Pr}_{\mathcal{D}_S}[y=1]\Big|
        % &= \sup_{h\in \mathcal{H}} \Big| \Big(\text{Pr}_{\mathcal{D}_S}[I(h)] - \text{Pr}_{\mathcal{D}_T}[I(h)] \Big) \cdot \mathcal{S}_h \Big| + \Big|\text{Pr}_{\mathcal{D}_T}[y=1] - \text{Pr}_{\mathcal{D}_S}[y=1]\Big|
    \end{aligned}
\end{equation*}
which completes the proof.
\end{proof}

{\bf Proof of Lemma~\ref{L: negative_transfer}}.
Lemma~\ref{L: negative_transfer} states that when $p_S(\mathbf{x})=p_T(\mathbf{x})$ and $\epsilon_S(h)=0$, if $\mathcal{L}(h(\mathbf{x}),y)=|h(\mathbf{x})-y|$, we have
\begin{equation*}
\begin{aligned}
    \epsilon_T(h) \geq \big| \text{Pr}_{\mathcal{D}_T}[y=1] - \text{Pr}_{\mathcal{D}_S}[y=1] \big|
\end{aligned}
\end{equation*}
\begin{proof}
We know that $\text{Pr}_{\mathcal{D}_S}[y=1] = \int p_S(\mathbf{x},y=1) d\mathbf{x} = \int \left(\sum_y p_S(\mathbf{x},y) y \right) d\mathbf{x}$, then,
\begin{equation*}
    \begin{aligned}
        \big|\text{Pr}_{\mathcal{D}_S}[y=1] - \int p_S(\mathbf{x})h(\mathbf{x}) d\mathbf{x} \big|
        &= \left|\int\left(\sum_y p_S(\mathbf{x},y) y \right)d\mathbf{x} - \int\left(\sum_y p_S(\mathbf{x},y)h(\mathbf{x}) \right) d\mathbf{x} \right| \\
        \leq & \int\left(\sum_y p_S(\mathbf{x},y) |y-h(\mathbf{x})| \right) d\mathbf{x} = \epsilon_S(h)=0
    \end{aligned}
\end{equation*}
Thus, $\text{Pr}_{\mathcal{D}_S}[y=1] = \int p_S(\mathbf{x})h(\mathbf{x}) d\mathbf{x}$, and
\begin{equation*}
    \begin{aligned}
        \epsilon_T(h) &= \int\left(\sum_y p_T(\mathbf{x},y) \mathcal{L} \left(h(\mathbf{x}), y \right)  \right) d\mathbf{x} = \int\left(\sum_y p_T(\mathbf{x},y) |h(\mathbf{x})-y| \right) d\mathbf{x} \\
        & \geq \Big|\int\left(\sum_y p_T(\mathbf{x},y) (h(\mathbf{x})-y) \right) d\mathbf{x} \Big| = \Big|\int p_T(\mathbf{x})h(\mathbf{x}) d\mathbf{x} - \int\left(\sum_y p_T(\mathbf{x},y) y \right) d\mathbf{x} \Big| \\
        & = \Big|\int p_S(\mathbf{x})h(\mathbf{x}) d\mathbf{x} - \text{Pr}_{\mathcal{D}_T}[y=1] \Big| = \Big| \text{Pr}_{\mathcal{D}_S}[y=1] - \text{Pr}_{\mathcal{D}_T}[y=1] \Big|
    \end{aligned}
\end{equation*}
which completes the proof.
\end{proof}

{\bf Proof of Theorem~\ref{generalization}}.
Theorem~\ref{generalization} states that if loss function $\mathcal{L}$ is bounded, i.e., there exists $M>0$ such that $0\leq\mathcal{L} \leq M$, for a hypothesis $h\in \mathcal{H}$,  the target error can be bounded by the source error and the $\mathcal{C}$-divergence between the distributions $\mathcal{D}_S$ and $\mathcal{D}_T$. Specifically, we have
\begin{equation*}
    \epsilon_T(h) \leq \epsilon_S(h) + M\cdot d_{\mathcal{C}}(\mathcal{D}_S, \mathcal{D}_T)
\end{equation*}
\begin{proof}
Given $\epsilon_S(h) = \mathbb{E}_{(\mathbf{x},y)\sim \mathcal{D}_S} \big[\mathcal{L}(h(\mathbf{x}), y)|\big]$, we have
\begin{equation*}
    \begin{aligned}
    \epsilon_T(h) &= \epsilon_S(h) + \epsilon_T(h) - \epsilon_S(h) \\
    &\leq \epsilon_S(h) + \Big|\text{Pr}_{\mathcal{D}_S}[\mathcal{L}(h(\mathbf{x}), y)] - \text{Pr}_{\mathcal{D}_T}[\mathcal{L}(h(\mathbf{x}), y)] \Big| \\
    &\leq \epsilon_S(h) + M\cdot \Big|\text{Pr}_{\mathcal{D}_S}[h(\mathbf{x}) \neq y] - \text{Pr}_{\mathcal{D}_T}[h(\mathbf{x}) \neq y] \Big| \\
    & = \epsilon_S(h) + M\cdot \Big|\text{Pr}_{\mathcal{D}_S}[h(\mathbf{x}) = y] - \text{Pr}_{\mathcal{D}_T}[h(\mathbf{x}) = y] \Big| \\
    &= \epsilon_S(h) + M\cdot \Big|\text{Pr}_{\mathcal{D}_S}[\{I(h), y=1\}\cup \{\overline{I(h)}, y=0\}] - \text{Pr}_{\mathcal{D}_T}[\{I(h), y=1\}\cup \{\overline{I(h)}, y=0\}] \Big| \\
    &\leq \epsilon_S(h) + M\cdot d_{\mathcal{C}}(\mathcal{D}_S, \mathcal{D}_T)
    \end{aligned}
\end{equation*}
which completes the proof.
\end{proof}

{\bf Proof of Lemma~\ref{T: Rademacher_bound}}.
Lemma~\ref{T: Rademacher_bound} states that for any $\delta \in (0,1)$, with probability at least $1-\delta$ over $m_S$ labeled source samples $\mathcal{B}_S$ and $m_T$ labeled target samples $\mathcal{B}_T$, we have:
\begin{equation*}
    \begin{aligned}
    d_\mathcal{C}(\mathcal{D}_S, \mathcal{D}_T) &\leq d_{\mathcal{C}}(\hat{\mathcal{D}}_S, \hat{\mathcal{D}}_T) + \Big(\hat{\Re}_{\mathcal{B}_S}(L_H) + \hat{\Re}_{\mathcal{B}_T}(L_H) \Big)  + 3 \Bigg(\sqrt{\frac{\log{\frac{4}{\delta}}}{2m_S}} + \sqrt{\frac{\log{\frac{4}{\delta}}}{2m_T}} \Bigg)
    \end{aligned}
\end{equation*}
\begin{proof}
Based on the Rademacher Bound~\cite{mansour2009domain}, with probability at least $1-\delta/2$ over $m_S$ labeled source samples $\mathcal{B}_S$, we have
\begin{equation*}
\begin{aligned}
    \mathbb{E}_{(\mathbf{x},y)\sim p_S(\mathbf{x},y)}[h(\mathbf{x})=y] & \leq \mathbb{E}_{(\mathbf{x},y)\sim \hat{p}_S(\mathbf{x},y)}[h(\mathbf{x})=y] + \hat{\Re}_{\mathcal{B}_S}(L_H) + 3\sqrt{\frac{\log{\frac{4}{\delta}}}{2m_S}}
\end{aligned}
\end{equation*}
where $\hat{p}_S(\mathbf{x},y)$ is the empirical estimated probability density function on source domain.
Since $\text{Pr}_{\mathcal{D}_S}[h(\mathbf{x})=y] = \mathbb{E}_{(\mathbf{x},y)\sim p_S(\mathbf{x},y)}[h(\mathbf{x})=y]$ for any $h\in \mathcal{H}$. Thus,
\begin{equation*}
    d_{\mathcal{C}}(\mathcal{D}_S, \hat{\mathcal{D}}_S) \leq \hat{\Re}_{\mathcal{B}_S}(L_H) + 3\sqrt{\frac{\log{\frac{4}{\delta}}}{2m_S}}
\end{equation*}
The same result holds for target domain. Based on the triangle inequality,
\begin{equation*}
    \begin{aligned}
        d_\mathcal{C}(\mathcal{D}_S, \mathcal{D}_T) &\leq d_{\mathcal{C}}(\mathcal{D}_S, \hat{\mathcal{D}}_S) + d_{\mathcal{C}}(\hat{\mathcal{D}}_S, \hat{\mathcal{D}}_T) + d_{\mathcal{C}}(\mathcal{D}_T, \hat{\mathcal{D}}_T) \\
        &\leq d_{\mathcal{C}}(\hat{\mathcal{D}}_S, \hat{\mathcal{D}}_T) + \Big(\hat{\Re}_{\mathcal{B}_S}(L_H) + \hat{\Re}_{\mathcal{B}_T}(L_H) \Big)  + 3 \Bigg(\sqrt{\frac{\log{\frac{4}{\delta}}}{2m_S}} + \sqrt{\frac{\log{\frac{4}{\delta}}}{2m_T}} \Bigg)
    \end{aligned}
\end{equation*}
which completes the proof.
\end{proof}

{\bf Proof of Theorem~\ref{empirical_generalization_bound}}.
Theorem \ref{empirical_generalization_bound} states that assume loss function $\mathcal{L}$ is bounded with $0\leq\mathcal{L} \leq M$. For a hypothesis $h\in \mathcal{H}$ and $\delta \in (0,1)$, with probability at least $1-\delta$ over $m_S$ examples $\mathcal{B}_S$ drawn from $\mathcal{D}_S$ and $m_T$ examples $\mathcal{B}_T$ drawn from $\mathcal{D}_T$, we have:
\begin{equation*} 
\small
    \begin{aligned}
    \epsilon_T(h) &\leq \hat{\epsilon}_S(h) + M\Bigg(d_{\mathcal{C}}(\hat{\mathcal{D}}_S, \hat{\mathcal{D}}_T) + \hat{\Re}_{\mathcal{B}_S}(L_H) + \hat{\Re}_{\mathcal{B}_T}(L_H) + 3\sqrt{\frac{\log{\frac{8}{\delta}}}{2m_S}} + 3\sqrt{\frac{\log{\frac{8}{\delta}}}{2m_T}} + \sqrt{\frac{M^2\log{\frac{4}{\delta}}}{2m_S}} \Bigg)
    \end{aligned}
\end{equation*}
where $\hat{\epsilon}_S(h)$ denotes the empirical source error over finite data set $\mathcal{B}_S$.
\begin{proof}
Combining Theorem \ref{generalization}, Lemma \ref{T: Rademacher_bound} and Lemma \ref{estimation_error}, the result can be derived.
\end{proof}

{\bf Proof of Theorem~\ref{continuous_transfer_learning_bound}}.
Theorem \ref{continuous_transfer_learning_bound} states that assume the loss function $\mathcal{L}$ is bounded and $d_{\mathcal{C}}(\mathcal{D}_S, \mathcal{D}_{T_1}) \leq \Delta$, $d_{\mathcal{C}}(\mathcal{D}_{T_{i}}, \mathcal{D}_{T_{i+1}}) \leq \Delta$ for all $i=1,\cdots,n$ where $\Delta>0$. Then, for any $\delta>0$ and $h\in \mathcal{H}$, with probability at least $1-\delta$, the target domain error $\epsilon_{T_{t+1}}$ is bounded by
\begin{equation*}
    \begin{aligned}
        &\epsilon_{T_{t+1}}(h) \leq \frac{1}{t+1} \left( \hat{\epsilon}_S(h) + \sum_{i=1}^t \hat{\epsilon}_{T_i}(h) \right) + \frac{(t+2)M\Delta}{2} + \Tilde{\delta}
    \end{aligned}
\end{equation*}
where $\Tilde{\delta}=\frac{M}{t+1}\left( \sqrt{\frac{\log{\frac{2(t+1)}{\delta}}}{2m_S}} +\sum_{i=1}^t \sqrt{\frac{\log{\frac{2(t+1)}{\delta}}}{2m_{T_i}}} + \sqrt{\frac{2\log{\frac{2}{\delta}}}{m_{all}}} \right)$, $m_{all} = m_S + \sum_{i=1}^t m_{T_i}$ and $m_{T_i}$ is the number of labeled instances in $\mathcal{D}_{T_i}$.
\begin{proof}
For any sample set
$\mathcal{B}_{S,T}=\left( \{(\mathbf{x}_j, y_j)\}_{j=1}^{m_S}, \{(\mathbf{x}_j, y_j)\}_{j=1}^{m_{T_1}}, \cdots, \{(\mathbf{x}_j, y_j)\}_{j=1}^{m_{T_t}} \right) \in \left( \mathcal{X}\times \mathcal{Y} \right)^{m_{all}}$ sampled from the product distribution $p(\mathbf{x},y)=p_S(\mathbf{x},y)^{m_S} \otimes p_{T_1}(\mathbf{x},y)^{m_{T_1}} \otimes \cdots \otimes p_{T_t}(\mathbf{x},y)^{m_{T_t}}$, we define a function $g$ over $\mathcal{B}_{S,T}$ as follows.
\begin{equation*}
    g(\mathcal{B}_{S,T}) = \epsilon_{T_{t+1}}(h) - \frac{1}{t+1} \left( \hat{\epsilon}_S(h) + \sum_{i=1}^t \hat{\epsilon}_{T_i}(h) \right)
\end{equation*}
where $\hat{\epsilon}_S(h)=\frac{1}{m_S} \sum_{j=1}^{m_S} \mathcal{L}\left(h(\mathbf{x}_j), y_j\right)$ and $\hat{\epsilon}_{T_i}(h)=\frac{1}{m_{T_i}} \sum_{j=1}^{m_{T_i}} \mathcal{L}\left(h(\mathbf{x}_j), y_j\right)$ for all $i=1,\cdots,n$.

Let $\mathcal{B}_{S,T}$ and $\mathcal{B}'_{S,T}$ be two sample sets containing only one different labeled sample, then we have
\begin{equation*}
    \left| g(\mathcal{B}_{S,T}) - g(\mathcal{B}'_{S,T}) \right| \leq \frac{1}{t+1} \left| \mathcal{L}\left(h(\mathbf{x}_j), y_j \right) - \mathcal{L}\left(h(\mathbf{x}'_j), y'_j\right) \right| \leq \frac{2M}{t+1}
\end{equation*}
Based on McDiarmid's inequality (showin in Lemma \ref{McDiarmid}), we have for any $\epsilon > 0$
\begin{equation*}
    \text{Pr}\left[ \left| g(\mathcal{B}_{S,T}) -\mathbb{E}_{\mathcal{B}_{S,T}\sim p(\mathbf{x},y)} \left[g(\mathcal{B}_{S,T}) \right] \right| \geq \epsilon \right] \leq \exp{\left( \frac{-m_{all}(t+1)^2\epsilon^2}{2M^2} \right)}
\end{equation*}
Then, for any $\delta/2 > 0$, with probability at least $1-\delta/2$, the following holds
\begin{equation*}
    g(\mathcal{B}_{S,T}) \leq \mathbb{E}_{\mathcal{B}_{S,T}\sim p(\mathbf{x},y)} \left[g(\mathcal{B}_{S,T})\right] + \frac{M}{t+1}\sqrt{\frac{2\log{\frac{2}{\delta}}}{m_{all}}}
\end{equation*}
Besides, based on Lemma \ref{estimation_error} and triangle equality of $\mathcal{C}$-divergence, for any $\delta/2 > 0$, with probability at least $1-\delta/2$, we have
\begin{equation*}
    \begin{aligned}
        &\mathbb{E}_{\mathcal{B}_{S,T}\sim p(\mathbf{x},y)} \left[g(\mathcal{B}_{S,T})\right] = \mathbb{E}_{\mathcal{B}_{S,T}\sim p(\mathbf{x},y)} \left[ \epsilon_{T_{t+1}}(h) - \frac{1}{t+1} \left( \hat{\epsilon}_S(h) + \sum_{i=1}^t \hat{\epsilon}_{T_i}(h) \right) \right] \\
        =& \mathbb{E}_{\mathcal{B}_{S,T}\sim p(\mathbf{x},y)} \left[ \epsilon_{T_{t+1}}(h) - \frac{1}{t+1} \left( {\epsilon}_S(h) + \sum_{i=1}^t {\epsilon}_{T_i}(h) \right) \right] \\
        & + \mathbb{E}_{\mathcal{B}_{S,T}\sim p(\mathbf{x},y)} \left[
        \frac{1}{t+1} \left( {\epsilon}_S(h) + \sum_{i=1}^t {\epsilon}_{T_i}(h) \right) - \frac{1}{t+1} \left( \hat{\epsilon}_S(h) + \sum_{i=1}^t \hat{\epsilon}_{T_i}(h) \right) \right] \\
        =& \frac{1}{t+1} \left[ \left( \epsilon_{T_{t+1}}(h) - {\epsilon}_S(h) \right) + \sum_{i=1}^t \left( \epsilon_{T_{t+1}}(h) - {\epsilon}_{T_i}(h) \right) \right] \\
        & + \frac{1}{t+1} \mathbb{E}_{\mathcal{B}_{S,T}\sim p(\mathbf{x},y)} \left[ \left( \epsilon_S(h) - \hat{\epsilon}_S(h) \right) + \sum_{i=1}^t \left( \epsilon_{T_i}(h)- \hat{\epsilon}_{T_i}(h) \right) \right] \\
        \leq & \frac{1}{t+1} \left( d_{\mathcal{C}}(\mathcal{D}_{T_{t+1}}, \mathcal{D}_S) + \sum_{i=1}^t d_{\mathcal{C}}(\mathcal{D}_{T_{t+1}}, \mathcal{D}_{T_i}) \right) + \frac{M}{t+1}\left( \sqrt{\frac{\log{\frac{2(t+1)}{\delta}}}{2m_S}} +\sum_{i=1}^t \sqrt{\frac{\log{\frac{2(t+1)}{\delta}}}{2m_{T_i}}} \right) \\
        \leq & \frac{1}{t+1} \left( (t+1)\Delta + \sum_{i=1}^t (t+1-i) \Delta \right) + \frac{M}{t+1}\left( \sqrt{\frac{\log{\frac{2(t+1)}{\delta}}}{2m_S}} +\sum_{i=1}^t \sqrt{\frac{\log{\frac{2(t+1)}{\delta}}}{2m_{T_i}}} \right) \\
        \leq & \frac{(t+2)M\Delta}{2} + \frac{M}{t+1}\left( \sqrt{\frac{\log{\frac{2(t+1)}{\delta}}}{2m_S}} +\sum_{i=1}^t \sqrt{\frac{\log{\frac{2(t+1)}{\delta}}}{2m_{T_i}}} \right)
    \end{aligned}
\end{equation*}
Therefore,
\begin{equation*}
    \begin{aligned}
        \epsilon_{T_{t+1}}(h) & \leq \frac{1}{t+1} \left( \hat{\epsilon}_S(h) + \sum_{i=1}^t \hat{\epsilon}_{T_i}(h) \right) + \mathbb{E}_{\mathcal{B}_{S,T}\sim p(\mathbf{x},y)} \left[f(\mathcal{B}_{S,T})\right] + \frac{M}{t+1}\sqrt{\frac{2\log{\frac{2}{\delta}}}{m_{all}}} \\
        &\leq \frac{1}{t+1} \left( \hat{\epsilon}_S(h) + \sum_{i=1}^t \hat{\epsilon}_{T_i}(h) \right) + \frac{(t+2)M\Delta}{2} \\
        &\quad + \frac{M}{t+1}\left( \sqrt{\frac{\log{\frac{2(t+1)}{\delta}}}{2m_S}} +\sum_{i=1}^t \sqrt{\frac{\log{\frac{2(t+1)}{\delta}}}{2m_{T_i}}} \right) + \frac{M}{t+1}\sqrt{\frac{2\log{\frac{2}{\delta}}}{m_{all}}}
    \end{aligned}
\end{equation*}
which completes the proof.
\end{proof}

{\bf Proof of Theorem~\ref{T: semi_bound}}.
Theorem~\ref{T: semi_bound} states that if loss function $\mathcal{L}$ is bounded, let $\epsilon_{\alpha}(h) = \alpha \epsilon_T(h) + (1-\alpha) \epsilon_S(h)$, then we have
\begin{equation*}
\begin{aligned}
    \epsilon_T({h}_{\alpha}^*) &\leq \epsilon_T(h_T^*) + 2(1-\alpha) M d_{\mathcal{C}}(\mathcal{D}_S, \mathcal{D}_T)
\end{aligned}
\end{equation*}
Furthermore,
\begin{equation*}
    TS(\mathcal{D}_T || \mathcal{D}_S)) \leq  2(1-\alpha) M d_{\mathcal{C}}(\mathcal{D}_S, \mathcal{D}_T)
\end{equation*}
\begin{proof}
It is easy to show $|\epsilon_{\alpha}(h) - \epsilon_T(h)| = (1-\alpha)|\epsilon_T(h) - \epsilon_S(h)|\leq (1-\alpha)M\cdot d_{\mathcal{C}}(\mathcal{D}_S, \mathcal{D}_T)$. Then
\begin{equation*}
    \begin{aligned}
        \epsilon_T(h_{\alpha}^*) &\leq \epsilon_{\alpha}(h_{\alpha}^*) + (1-\alpha) M d_{\mathcal{C}}(\mathcal{D}_S, \mathcal{D}_T) \\
        &\leq \epsilon_{\alpha}(h_T^*) + (1-\alpha) M d_{\mathcal{C}}(\mathcal{D}_S, \mathcal{D}_T) \\
        & \leq \epsilon_T(h_T^*) + 2(1-\alpha) M d_{\mathcal{C}}(\mathcal{D}_S, \mathcal{D}_T) \\
    \end{aligned}
\end{equation*}
Then, the transfer signature can be bounded as follows.
\begin{equation*}
    \begin{aligned}
        TS(\mathcal{D}_T || \mathcal{D}_S)) &= \inf_{A\in \mathcal{G}} \Big( \epsilon_T \big(A(\mathcal{D}_S, \mathcal{D}_T)\big) - \epsilon_T \big(A(\emptyset, \mathcal{D}_T)\big) \Big) \\
        &=  \inf_{A\in \mathcal{G}} \big( \epsilon_T({h}_{\alpha}^*) - \epsilon_T(h_T^*) \big) \\
        &\leq \inf_{A\in \mathcal{G}} \big( 2(1-\alpha) M d_{\mathcal{C}}(\mathcal{D}_S, \mathcal{D}_T) \big) \\
        &= 2(1-\alpha) M d_{\mathcal{C}}(\mathcal{D}_S, \mathcal{D}_T)
    \end{aligned}
\end{equation*}
where both $M$ and $d_{\mathcal{C}}(\mathcal{D}_S, \mathcal{D}_T)$ are model-agnostic.
\end{proof}

\subsection{Proposed Framework}

\begin{figure}[h]
    \centering
    \includegraphics[width=\textwidth]{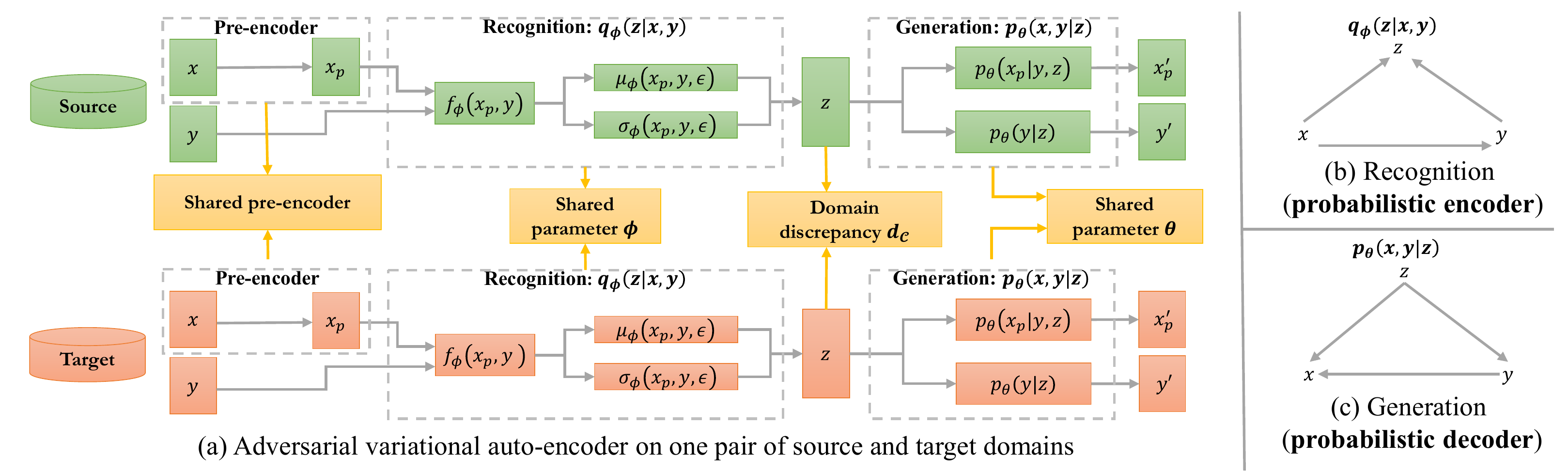}
    \caption{Overview of our proposed transfer learning framework (best viewed in color). (a) Adversarial variational auto-encoder learns domain-invariant hidden representation. (b) and (c) indicate the probabilistic graphical model for our recognition and generation modules.}\label{overview}
\end{figure}
Figure~\ref{overview} provides an overview of our proposed transfer learning framework based on label-informed $\mathcal{C}$-divergence. It can be seen that key components to our frameworks are variational auto-encoder and domain discrepancy measure. The intuition of variational auto-encoder used in our framework are as follows: (1) it learns a label-informed latent representation using both data feature and data label in order to estimate the C-divergence between source and target domains; (2) it could learn the discriminative classifier $q(\cdot|\mathbf{x})$ in a semi-supervised manner using knowledge from both labeled source examples and limited labeled target examples as well as adequate unlabeled target examples. Then, the domain discrepancy $d_{\mathcal{C}}$ could be estimated using the label-informed latent representation from source and target domains such that the minimization of $\mathcal{C}$-divergence $d_{\mathcal{C}}$ enables the better alignment of data distributions across domains. In addition, Figure \ref{overview}(b)(c) provides the probabilistic graphical model for our recognition (probabilistic encoder) and generation (probabilistic decoder) modules in our framework. It assumes that for probabilistic encoder $q_{\phi}(\mathbf{x},y, \mathbf{z})=q_{\phi}(\mathbf{z}| y, \mathbf{x})q_{\phi}(y|\mathbf{x})q(\mathbf{x})$, and for probabilistic decoder we have $p_{\theta}(\mathbf{x},y, \mathbf{z})=p_{\theta}(\mathbf{x}| y, \mathbf{z})p_{\theta}(y|\mathbf{z})p(\mathbf{z})$.

\subsection{Experimental Details}\label{appendix_ex}
We provide the experimental details, including data simulation, model configuration and additional results on digital image data sets. All our experiments are performed on a Windows machine with four 3.80GHz Intel Cores and 64GB RAM.

\subsubsection{Data Sets}
\begin{figure*}[h]
    \includegraphics[width=.32\textwidth]{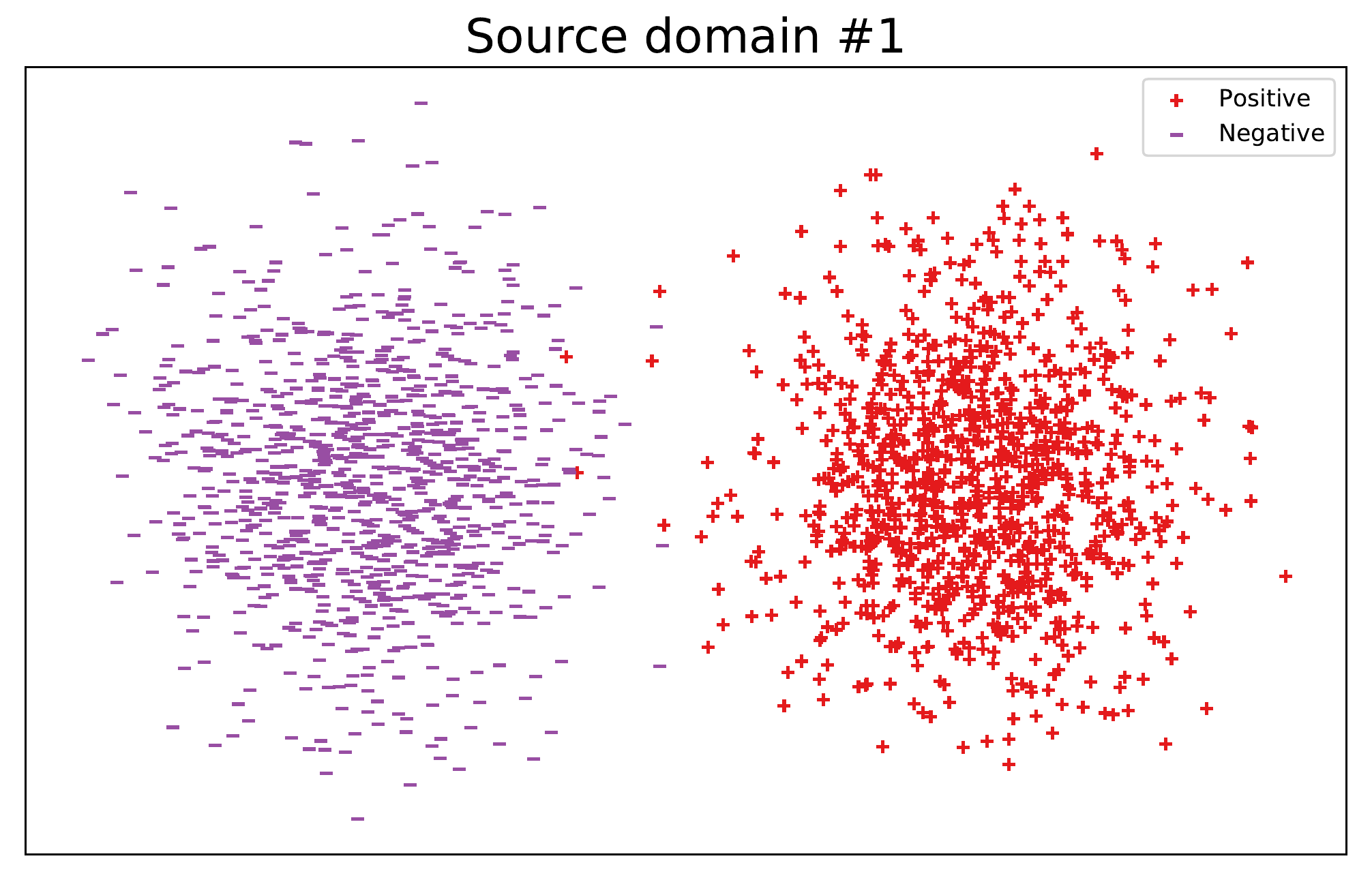}\hfill
    \includegraphics[width=.32\textwidth]{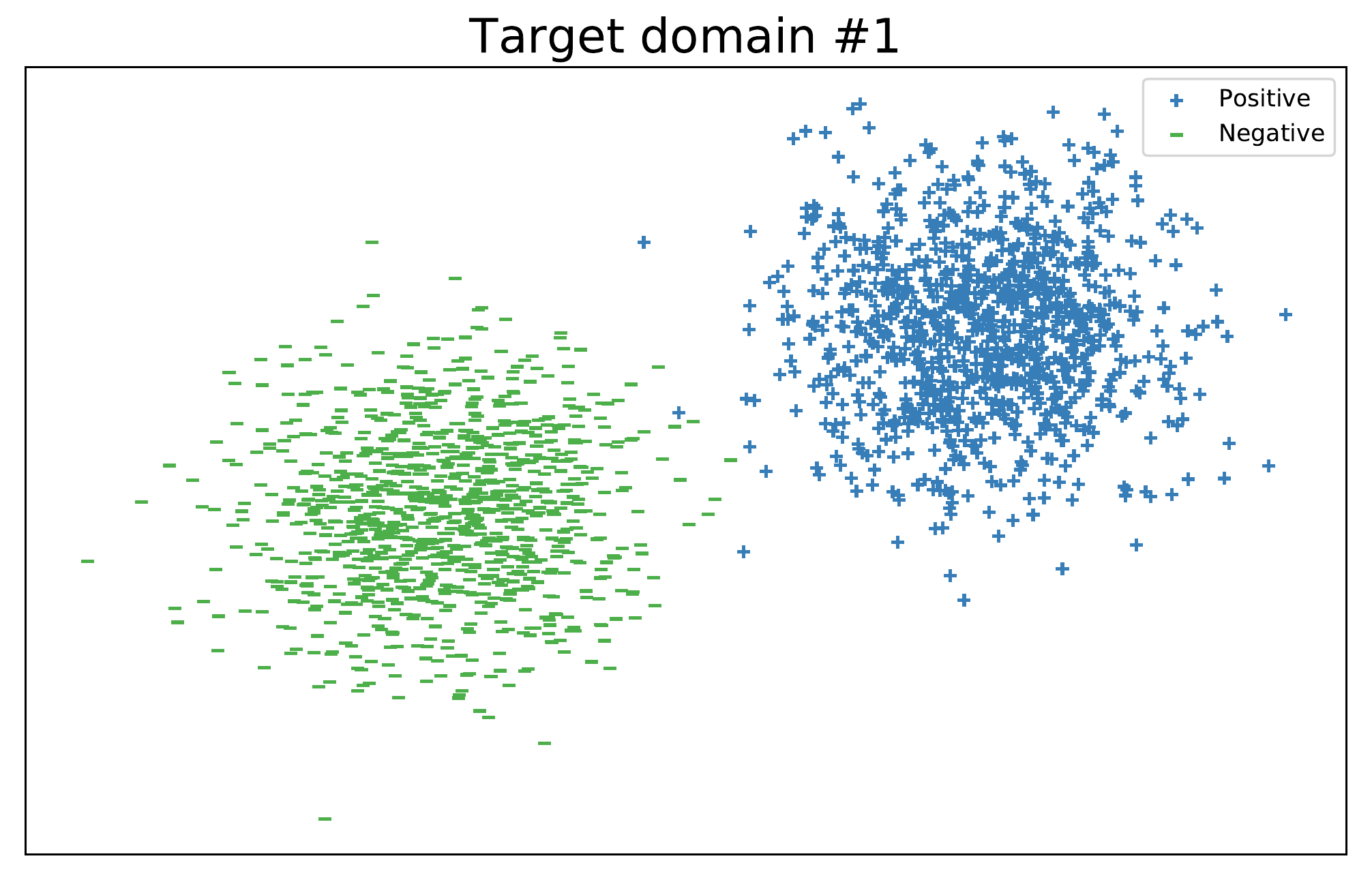}\hfill
    \includegraphics[width=.32\textwidth]{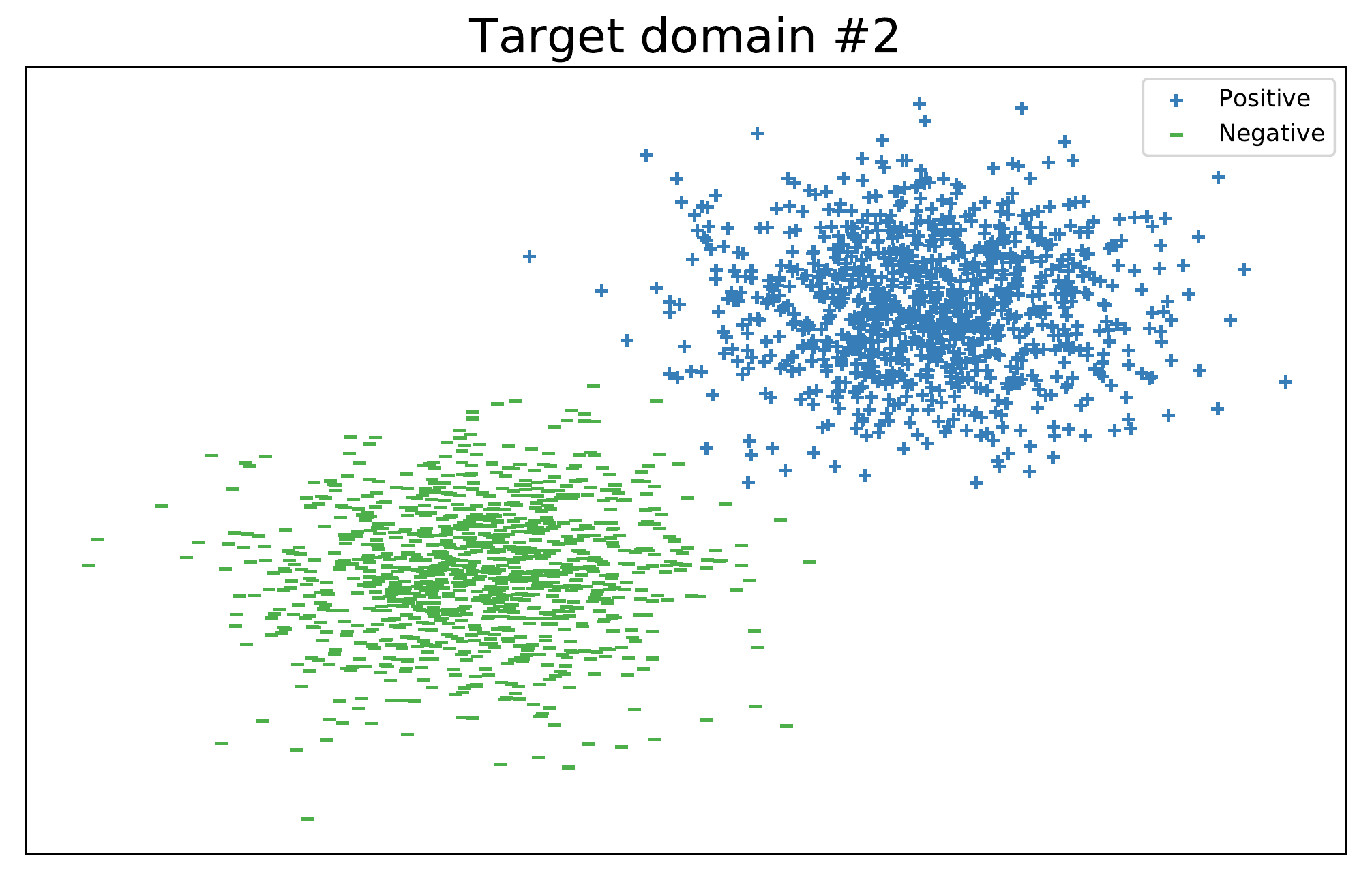}
    \\[\smallskipamount]
    \includegraphics[width=.32\textwidth]{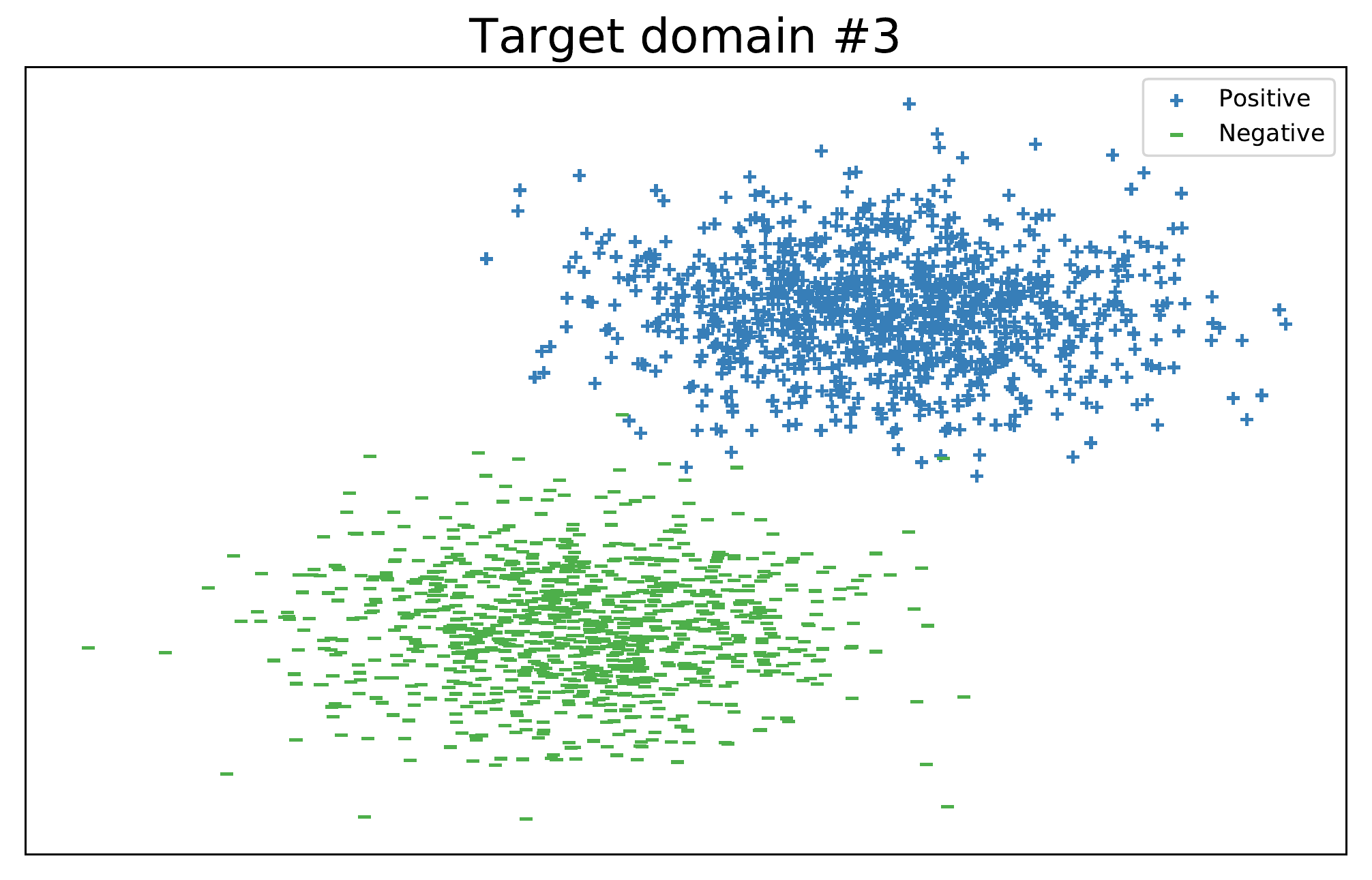}\hfill
    \includegraphics[width=.32\textwidth]{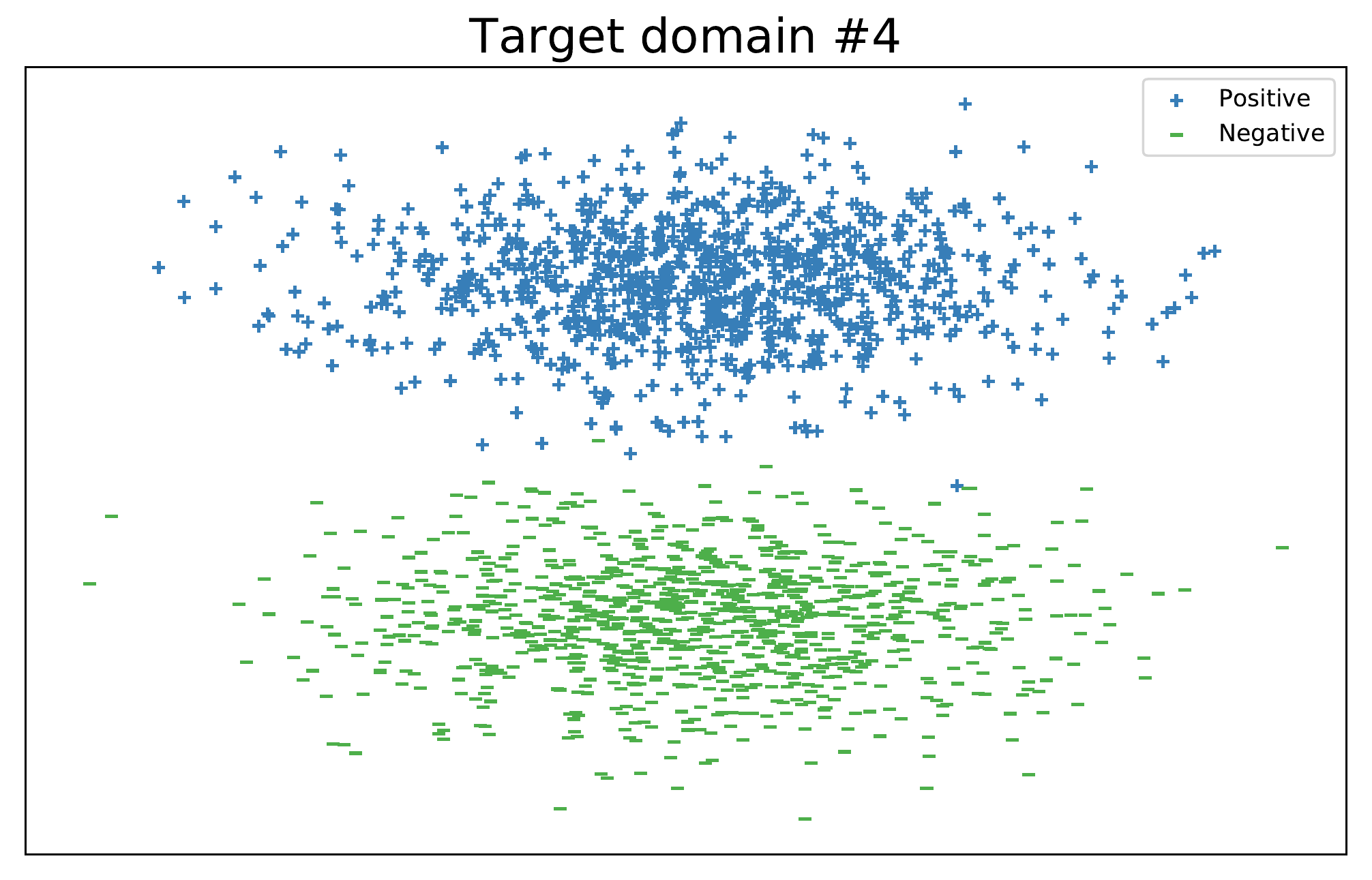}\hfill
    \includegraphics[width=.32\textwidth]{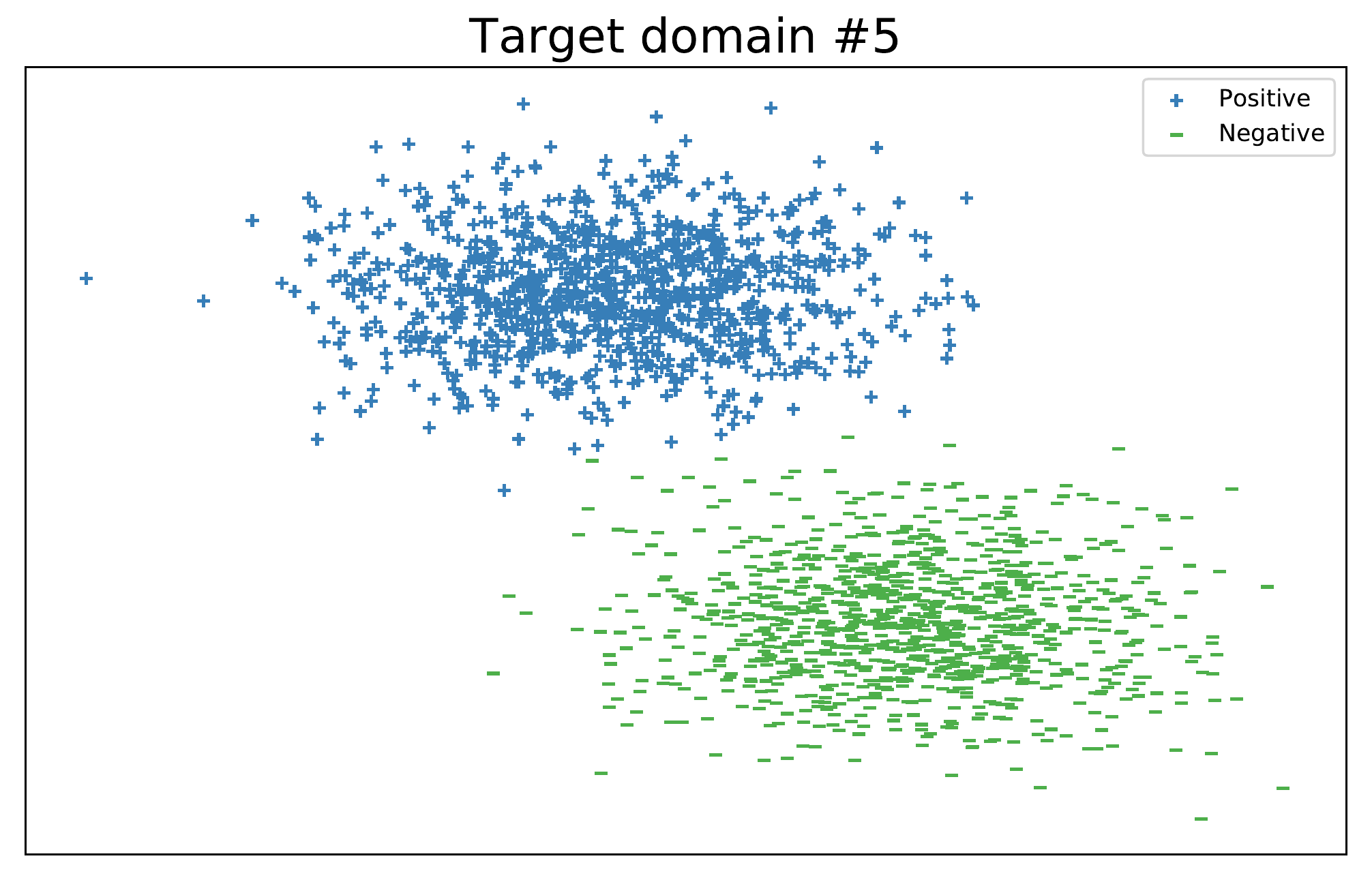}
    \\[\smallskipamount]
    \includegraphics[width=.32\textwidth]{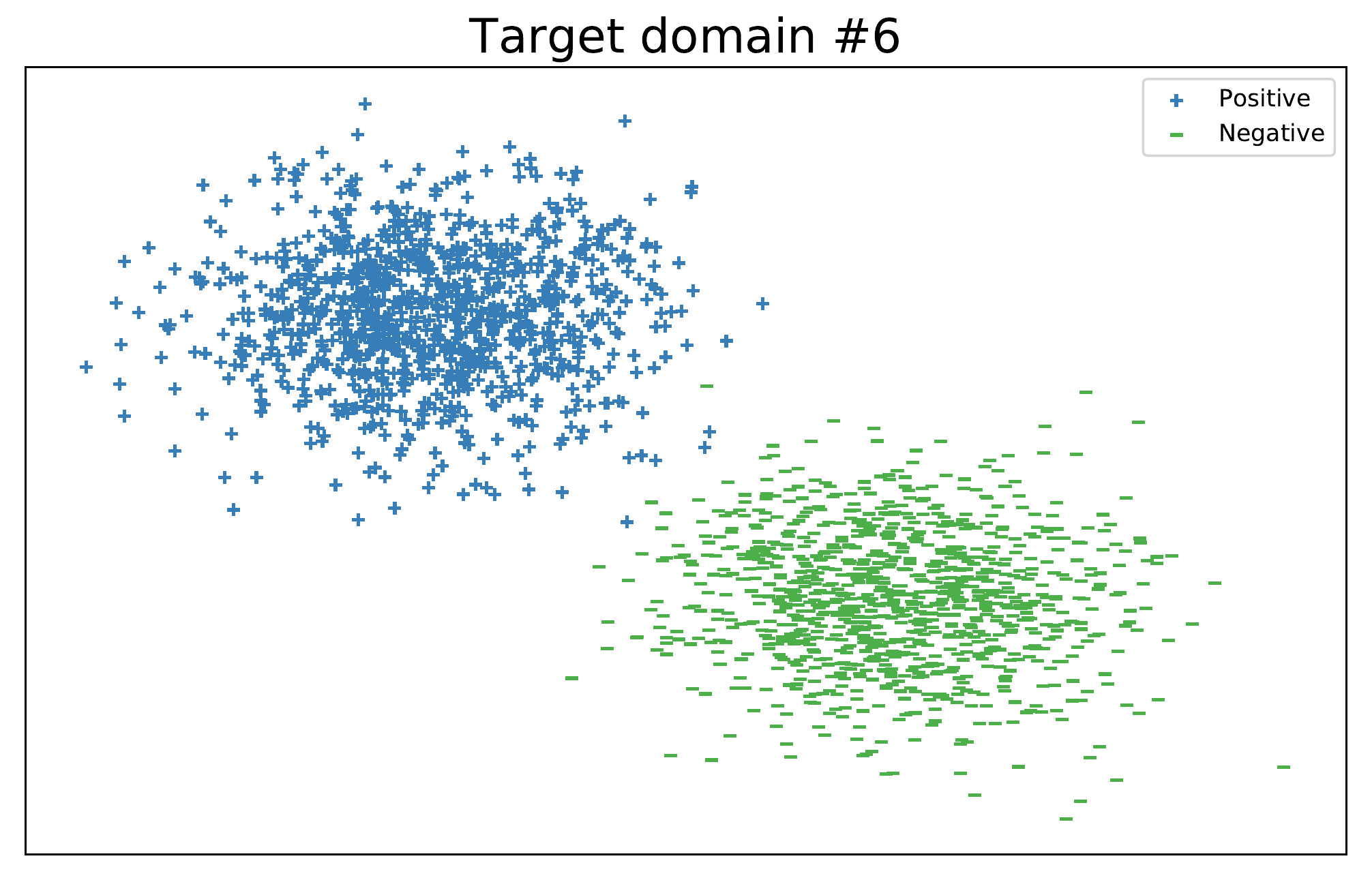}\hfill
    \includegraphics[width=.32\textwidth]{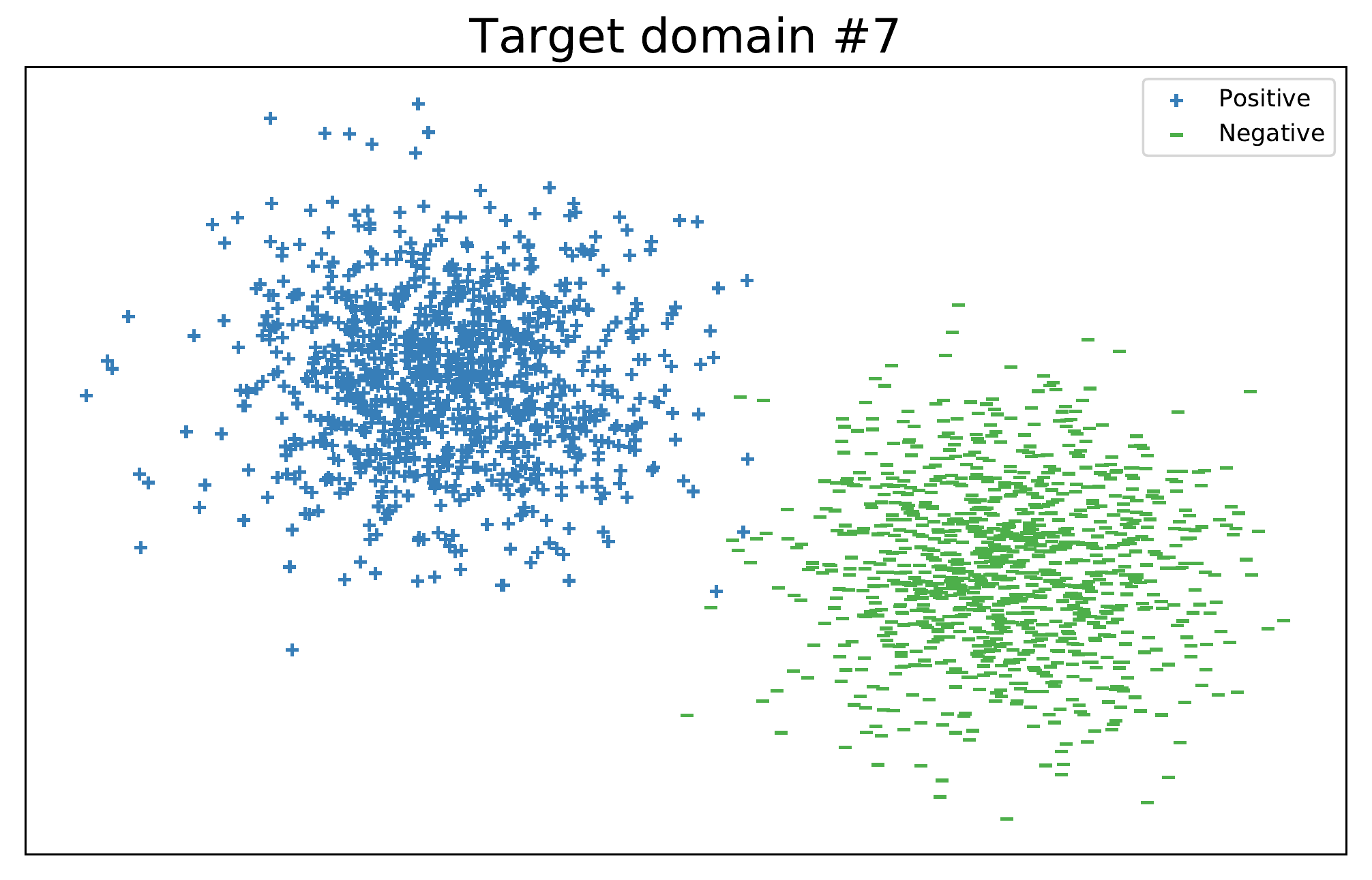}\hfill
    \includegraphics[width=.32\textwidth]{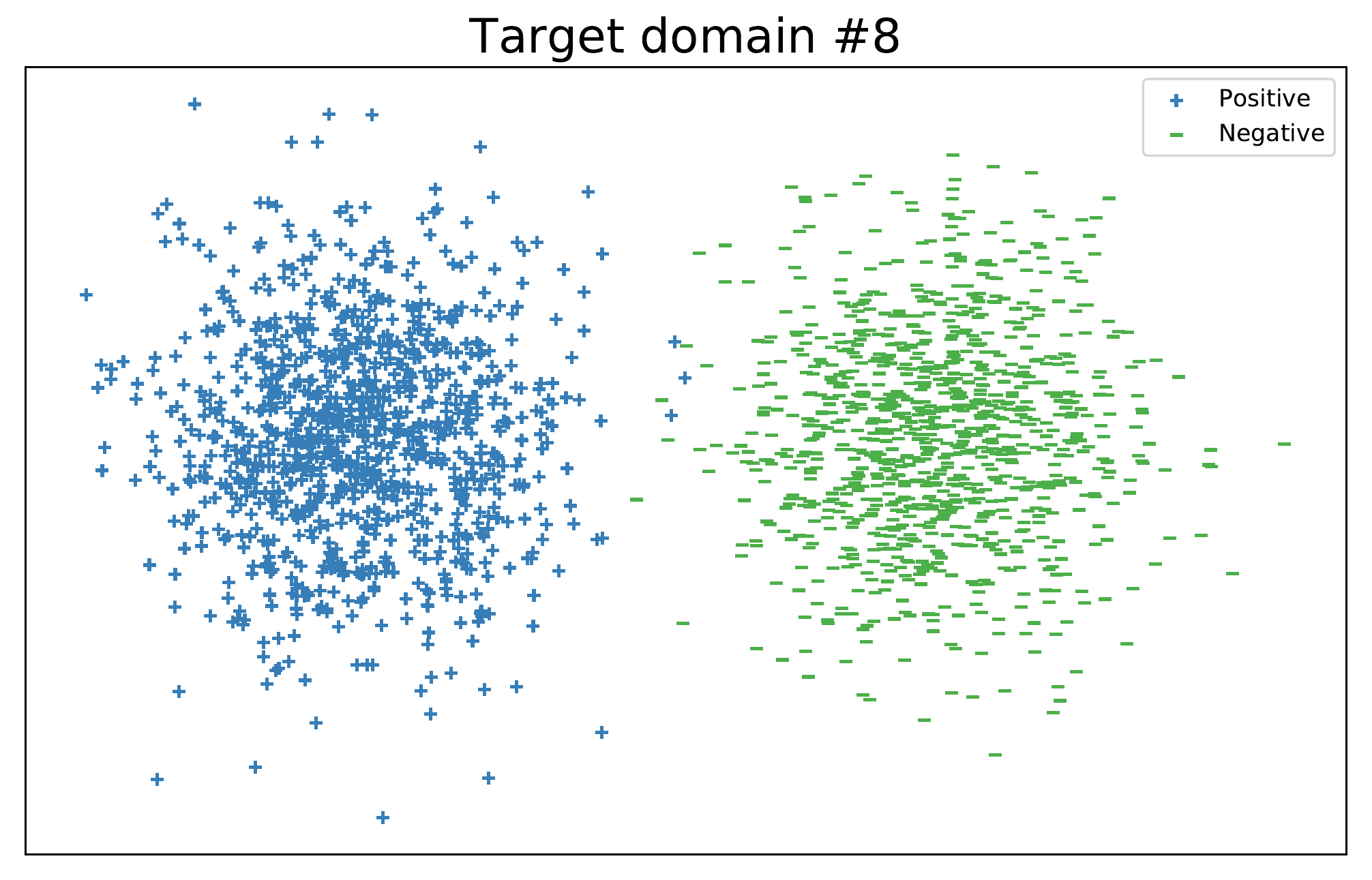}
    \caption{Synthetic source and target data (best viewed in color). For source domain (S1 at time stamp 1), positive samples are red ones and negative samples are violet ones. For target domain (T1, $\cdots$, T8 at time stamp 1$\sim$8), positive samples are in blue and negative samples are in green.}\label{fig:syn_data}
\end{figure*}

\textbf{Synthetic Data:} Figure \ref{fig:syn_data} provides the synthetic data set with a set of source and target data points where positive and negative samples are randomly sampled from two independent Gaussian distributions $\mathcal{N}([1.5\cos{\theta}, 1.5\sin{\theta}]^T, 0.5 \cdot \mathbf{I}_{2\times 2})$ and $\mathcal{N}([1.5\cos{(-\theta)}, 1.5\sin{(-\theta)}]^T, 0.5 \cdot \mathbf{I}_{2\times 2})$. We let $\theta=0$ for source domain (denoted as $S1$), and then the data points are rotated by setting $\theta$ as $\frac{\pi}{8}, \frac{\pi}{4}, \frac{3\pi}{8}, \frac{\pi}{2}, \frac{5\pi}{8}, \frac{3\pi}{4}, \frac{7\pi}{8}, \pi$ to generate the target domain with time-evolving nature. The data distribution of target domain slightly shifts in each time stamp. Intuitively, it can be observed that source domain S1 has the similar data distribution as the target domain T1, whereas it is significantly different from the target domain T8 (specifically, they have the significantly different conditional distribution $p(y|x)$ but similar marginal distribution $p(x)$).

\textbf{Real Data:} We used three publicly available data sets: MNIST\footnote{\url{http://yann.lecun.com/exdb/mnist/}} (with 60,000/10,000 train/test examples), SVHN\footnote{\url{http://ufldl.stanford.edu/housenumbers/}} (with 531,131/26,032 train/test examples) and USPS\footnote{\url{https://www.csie.ntu.edu.tw/~cjlin/libsvmtools/datasets/}} (with 7,291 / 2,007 train/test examples). In our experiments, we generate the time-evolving target domain by adding the adversarial noise to the clean target image data (e.g. MNIST for transfer learning on SVHN$\rightarrow$MNIST). The reason why we add the adversarial noise is that it could change the data distribution by adding the adversarial noise such that the generated adversarial examples largely fool the classifier learned on the clean examples. Besides, the generated adversarial examples are still highly separable in the new feature space, which has been empirically validated in our experiments by evaluating the TargetOnly method on those examples. More specifically, we used the Fast Gradient Sign Method (FGSM)~\cite{goodfellow2014explaining} to learn the adversarial noise on the image data sets. The adversarial noise generated by FGSM is defined as follows.
\begin{equation*}
    \tau = \omega \nabla_{\mathbf{x}} \mathcal{J}_{base}(\theta, \mathbf{x}, y)
\end{equation*}
where $\omega \geq 0$ is the magnitude of adversarial noise and $\mathcal{J}_{base}$ is the loss function of a neural network model (parameterized by $\theta$) to be attacked over example $(\mathbf{x}, y)$. Here we simply use the pre-trained LeNet\footnote{\url{https://drive.google.com/drive/folders/1fn83DF14tWmit0RTKWRhPq5uVXt73e0h}} model as the base model $\mathcal{J}_{base}$. Due to the transferability of adversarial examples, the adversarial examples generated by one model could easily fool another model. Therefore, give one target domain (e.g., MNIST for transfer learning on SVHN$\rightarrow$MNIST), we can generate new target domain examples by adding the adversarial noise. When the magnitude of adversarial noise $\omega$ linearly changes from 0.0 to 0.50 with an interval of 0.05, it would generate the evolving target domain examples. Figure \ref{svhn_mnist_show} shows the image examples of a static source domain (SVHN) and a time evolving target domain (MNIST) for continuous transfer learning.

\begin{figure*}[!t]
    \centering
    \includegraphics[width=\textwidth]{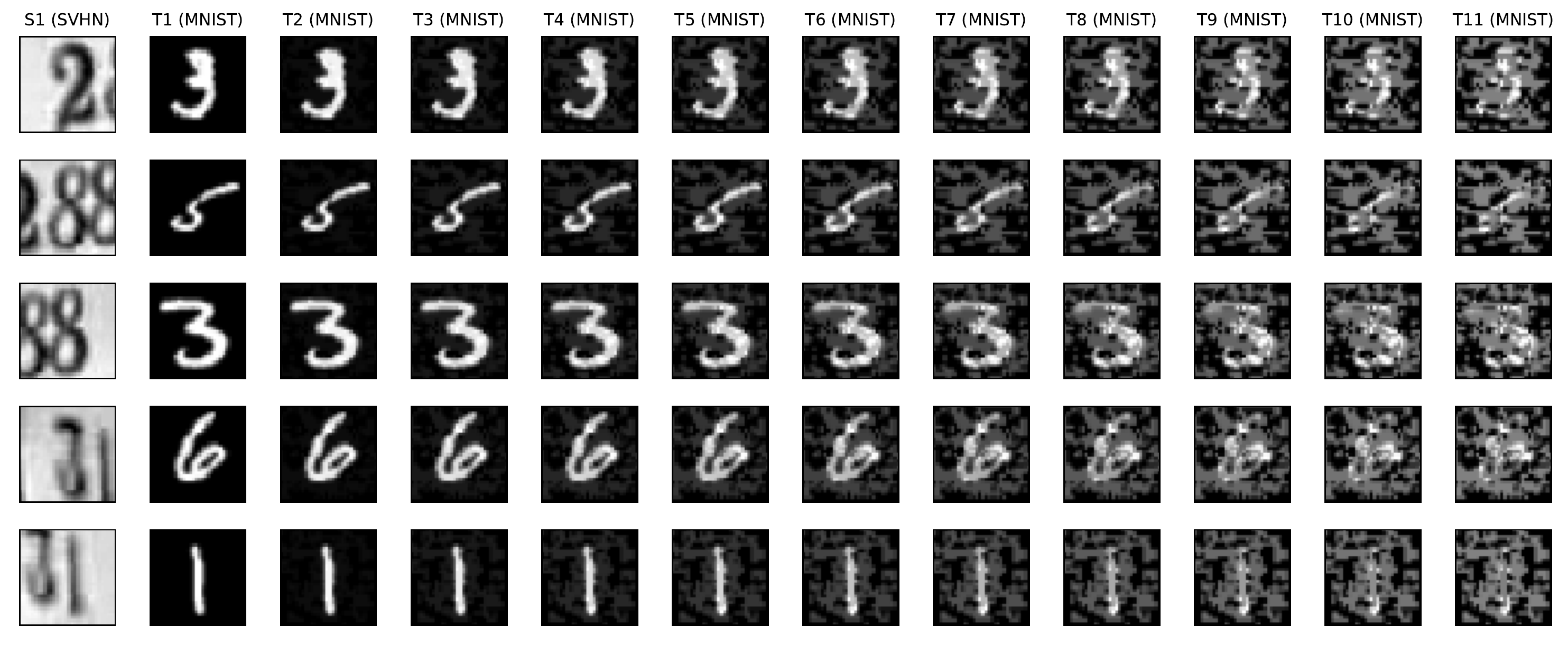}
    \vspace{-5mm}
    \caption{Examples of source domain (SVHN) and time-evolving target domain (MNIST). The first column is the source image examples in SVHN data set. The other columns are the target image examples from MNIST data set with different magnitude of adversarial noise.}\label{svhn_mnist_show}
\end{figure*}

% \begin{figure*}[!t]
%     \centering
%     \includegraphics[width=\textwidth]{figures/mnist_usps_show.pdf}
%     \vspace{-5mm}
%     \caption{Examples of source domain (MNIST) and time-evolving target domain (USPS). The first column is the source image examples from MNIST data set. The other columns are the target image examples from USPS data set with different magnitude of adversarial noise.}\label{mnist_usps_show}
% \end{figure*}

In addition, we consider another real transfer learning scenario where the source domain is SVHN and the evolving target domain is MNIST with various rotations as suggested in \cite{bobu2018adapting}. More specifically, the evolving target domain is generated by rotating the original MNIST images with rotation degree $0^{\circ}$, $15^{\circ}$, $30^{\circ}$, $45^{\circ}$, $60^{\circ}$, $75^{\circ}$ and $90^{\circ}$, respectively. Figure~\ref{svhn_mnist_rorate} shows the image examples of a static source domain (SVHN) and a time evolving target domain (MNIST).

\begin{figure*}[!t]
    \centering
    \includegraphics[width=0.6\textwidth]{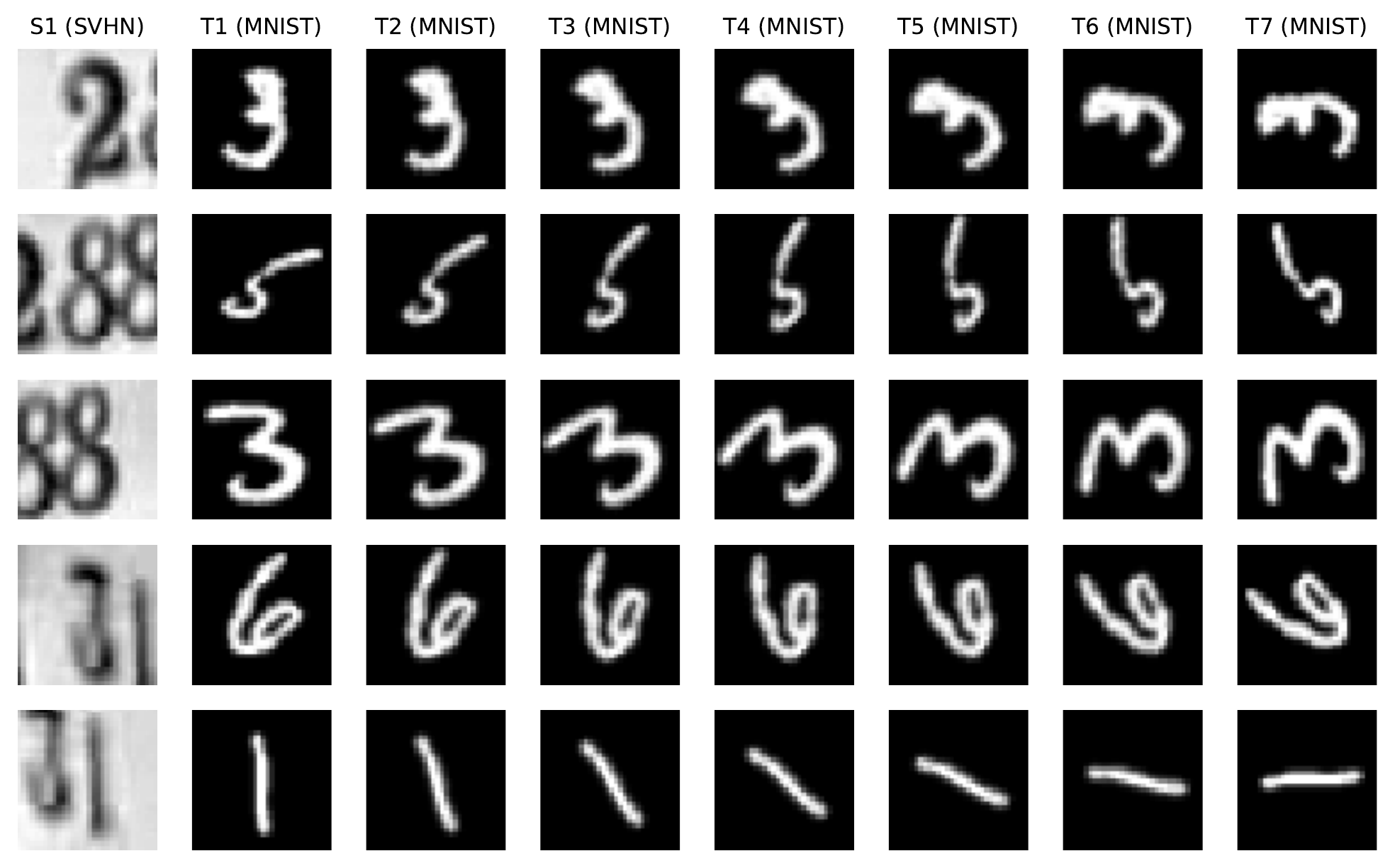}
    \caption{Examples of source domain (SVHN) and time-evolving target domain (MNIST). The first column is the source image examples from SVHN data set. The other columns are the target image examples from MNIST data set with different rotation degrees.}\label{svhn_mnist_rorate}
\end{figure*}

\subsubsection{Model Configuration}
The neural network architecture used in our experiments is shown in Figure \ref{architecture} where we used the gradient reversal layer (GRL)~\cite{ganin2016domain} to implement our proposed $\mathcal{C}$-divergence between source and target domains in the latent space.

In addition, we apply the Stochastic Gradient Descent (SGD) with the momentum of 0.9 to train our model where all the hidden parameters are initialized with Xavier initialization. The cross-entropy loss is adopted to measure the loss of label prediction and domain prediction. Following~\cite{ganin2016domain}, the learning rate $\eta_p$ is adjusted when training the model: $\eta_p=\frac{\eta_0}{(1+\alpha p)^{\beta}}$ where $p$ is an epoch-dependent scalar linearly varying from 0 to 1, and $\eta_0 = 0.01$, $\alpha=10$, $\beta=0.75$. The total number of training epochs is 10, 000 in our experiments. The domain adaptation parameter in gradient reversal layer is given by: $\lambda_p = \frac{2}{1+\exp{(-\gamma p)}} -1$ where $\gamma=10$.

\begin{figure*}[h]
    \centering
    \includegraphics[width=\textwidth]{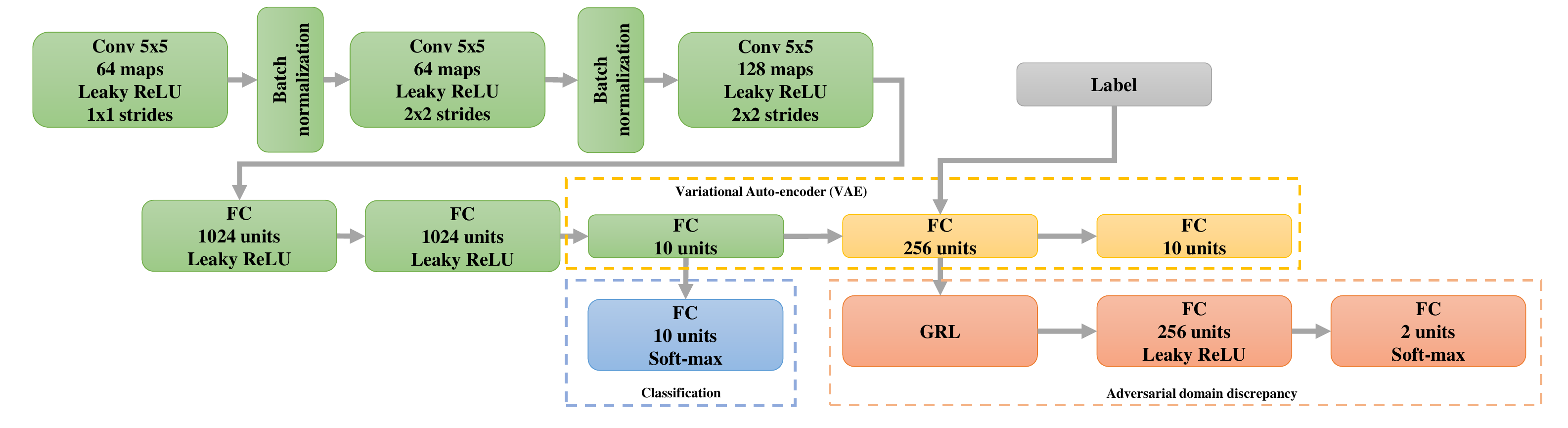}
    \vspace{-5mm}
    \caption{Neural network architecture used in our experiments. If data point is labeled, its class $y$ is used, otherwise, it uses the class prediction as a pseudo-label for learning the latent representation in the Variational Auto-encoder (VAE) framework. We applied the gradient reversal layer (GRL)~\cite{ganin2016domain} to implement the adversarial domain discrepancy.}\label{architecture}
\end{figure*}

\subsubsection{Additional Results}
{\bf Evaluation of Continuous Transfer Learning:}
Table~\ref{mnist_usps} and Table \ref{usps_mnist} shows the continuous transfer learning results on MNIST and USPS when using adversarial attacks to generate the evolving target domain. The results are consistent with our observations in Section \ref{experiments_CTL}. We would like to point out that TargetERM only is a semi-supervised learning scenario without using source and historical target knowledge when limited target examples are available and minimizes the empirical risk on those labeled examples. It can be seen that this baseline is stable, but could not achieve satisfactory performance on the evolving target domain. This tells us that (1) source and historical target knowledge could largely improve the classification performance on the evolving target domain; (2) static transfer learning baselines might produce worse classification performance than TargetERM, thus leading to the occurrence of negative transfer when data distribution between source and current target tasks are largely shifted for T6-T11.

\begin{table*}[!t]
\small
\centering
\setlength\tabcolsep{2.9pt}
\caption{Transfer learning accuracy from MNIST (source) to continuously evolving USPS (target)}\label{mnist_usps}
\begin{tabular}{|l|c|c|c|c|c|c|c|c|c|c|c|}
\hline
Target Domain & T1              & T2              & T3              & T4              & T5              & T6              & T7              & T8              & T9              & T10             & T11             \\ \hline\hline
SourceOnly    & 0.8196          & 0.7778          & 0.6946          & 0.5745          & 0.3921          & 0.2272          & 0.1579          & 0.0907          & 0.0613          & 0.0429          & 0.0289          \\ \hline
TargetOnly    & 0.9616          & 0.9522          & 0.9646          & 0.9771          & 0.9781          & 0.9806          & 0.9791          & 0.9865          & 0.9880          & 0.9880          & 0.9880          \\ \hline\hline
TargetERM & 0.8012 & 0.7474 & 0.6951 & 0.6557 & 0.6253 & 0.6412 & 0.7205 & 0.7384 & 0.7693 & 0.7828 & 0.8381 \\\hline
CORAL~\cite{sun2016return}         & 0.8570          & 0.8211          & 0.7897          & 0.7195          & 0.7240          & 0.6288          & 0.6323          & 0.6831          & 0.6313          & 0.6139          & 0.6632          \\ \hline
DANN~\cite{ganin2016domain}          & 0.9088          & 0.8774          & 0.8411          & 0.8037          & 0.7633          & 0.7389          & 0.7260          & 0.6413          & 0.6986          & 0.7688          & 0.7997          \\ \hline
ADDA~\cite{tzeng2017adversarial}          & 0.9098          & 0.8859          & 0.8540          & 0.8012          & 0.7210          & 0.5835          & 0.4509          & 0.4434          & 0.4245          & 0.4410          & 0.4808          \\ \hline
WDGRL~\cite{shen2017wasserstein}         & 0.9133          & 0.8485          & 0.8510          & 0.8067          & 0.7793          & 0.7195          & 0.7559          & 0.7369          & 0.8127          & 0.8052          & 0.8062          \\ \hline
DIFA~\cite{volpi2018adversarial}          & 0.8680          & 0.8361          & 0.8122          & 0.7683          & 0.7140          & 0.6163          & 0.4295          & 0.3687          & 0.4559          & 0.3627          & 0.4425          \\ \hline\hline
TransLATE\_p  & 0.9537          & \textbf{0.9392} & 0.9053          & 0.8655          & 0.8306          & 0.8176          & 0.7997          & 0.8550          & 0.8615          & 0.8445          & 0.8710          \\ \hline
TransLATE     & \textbf{0.9537} & 0.9367          & \textbf{0.9263} & \textbf{0.9178} & \textbf{0.9283} & \textbf{0.9352} & \textbf{0.9482} & \textbf{0.9517} & \textbf{0.9522} & \textbf{0.9591} & \textbf{0.9716} \\ \hline
\end{tabular}
% \vspace{-4mm}
\end{table*}

\begin{table*}[!t]
\small
\centering
\setlength\tabcolsep{2.9pt}
\caption{Transfer learning accuracy from USPS (source) to continuously evolving MNIST (target)}\label{usps_mnist}
\begin{tabular}{|l|c|c|c|c|c|c|c|c|c|c|c|}
\hline
Target Domain & T1              & T2              & T3              & T4              & T5              & T6              & T7              & T8              & T9              & T10             & T11             \\ \hline\hline
SourceOnly    & 0.4558          & 0.4429          & 0.4097          & 0.4105          & 0.3872          & 0.3437          & 0.3019          & 0.2472          & 0.2061          & 0.1506          & 0.1179          \\ \hline
TargetOnly    & 0.9934          & 0.9975          & 0.9992          & 0.9989          & 0.9989          & 0.9993          & 0.9989          & 0.9990          & 0.9987          & 0.9986          & 0.9988          \\ \hline\hline
TargetERM &0.7326& 0.6867& 0.6793& 0.6442& 0.6188& 0.6368&0.6386 & 0.7029& 0.7459 &0.7606& 0.7595 \\\hline
CORAL~\cite{sun2016return}         & 0.8619          & 0.8705          & 0.8259          & 0.7927          & 0.7161          & 0.6565          & 0.6753          & 0.6307          & 0.5930          & 0.6354          & 0.5783          \\ \hline
DANN~\cite{ganin2016domain}          & 0.8919          & 0.8791          & 0.8432          & 0.8188          & 0.7745          & 0.7535          & 0.7347          & 0.7298          & 0.6548          & 0.7088          & 0.6804          \\ \hline
ADDA~\cite{tzeng2017adversarial}          & 0.9130          & 0.8930          & 0.8432          & 0.7989          & 0.7436          & 0.7272          & 0.6569          & 0.6650          & 0.6114          & 0.6442          & 0.5431          \\ \hline
WDGRL~\cite{shen2017wasserstein}         & 0.9193          & 0.8841          & 0.8655          & 0.8031          & 0.7399          & 0.7010          & 0.7513          & 0.7309          & 0.7121          & 0.7570          & 0.7499          \\ \hline
DIFA~\cite{volpi2018adversarial}          & 0.9211          & 0.9151          & 0.8931          & 0.8408          & 0.8233          & 0.7408          & 0.6503          & 0.5433          & 0.3249          & 0.1728          & 0.0919          \\ \hline\hline
TransLATE\_p  & 0.9596          & 0.9197          & 0.9190          & 0.8539          & 0.8511          & 0.8856          & 0.8768          & 0.8898          & 0.8665          & 0.8738          & 0.8767          \\ \hline
TransLATE     & \textbf{0.9596} & \textbf{0.9493} & \textbf{0.9439} & \textbf{0.9333} & \textbf{0.9302} & \textbf{0.9243} & \textbf{0.9182} & \textbf{0.9143} & \textbf{0.9099} & \textbf{0.9161} & \textbf{0.9189} \\ \hline
\end{tabular}
\end{table*}

\begin{table}[!t]
\small
\centering
\caption{Transfer learning accuracy from SVHN (source) to continuously evolving MNIST (target) with various rotations}\label{tab:svhn_mnist_rotation}
\begin{tabular}{|l|c|c|c|c|c|c|c|}
\hline
             & T1 ($0^{\circ}$) & T2 ($15^{\circ}$) & T3 ($30^{\circ}$) & T4 ($45^{\circ}$) & T5 ($60^{\circ}$) & T6 ($75^{\circ}$) & T7 ($90^{\circ}$) \\ \hline
SourceOnly   & 0.6998 & 0.6879  & 0.6005  & 0.3135  & 0.1704  & 0.1340  & 0.1393  \\ \hline
% TargetERM    & 0.7270 & 0.7320  & 0.7372  & 0.7472  & 0.7464  & 0.7362  & 0.7306  \\ \hline
CORAL~\cite{sun2016return}        & 0.8349 & 0.8633  & 0.7527  & 0.6719  & 0.5969  & 0.5563  & 0.6155  \\ \hline
DANN~\cite{ganin2016domain}         & 0.8666 & 0.8332  & 0.7870  & 0.7606  & 0.6490  & 0.5799  & 0.6497  \\ \hline
WDGRL~\cite{shen2017wasserstein}        & 0.8990 & 0.8527  & 0.8290  & 0.8521  & 0.7582  & 0.8301  & 0.8126  \\ \hline
TransLATE\_p & 0.9621 & \textbf{0.9524}  & 0.8813  & 0.8543  & 0.7902  & 0.8564  & 0.8322  \\ \hline
TransLATE    & \textbf{0.9621} & 0.9469  & \textbf{0.8962}  & \textbf{0.8913}  & \textbf{0.8925}  & \textbf{0.9009}  & \textbf{0.8798}  \\ \hline
\end{tabular}
\end{table}
Table~\ref{tab:svhn_mnist_rotation} shows the transfer learning performance for a static source domain (SVHN) and a time-evolving target domain (MNIST with various rotation degrees). It demonstrated the effectiveness of our proposed \model{} algorithm on this data set. SourceOnly obtains terrible classification performance on the evolving target domain. In contrast, the transfer learning baselines could achieve significantly better performance when limited target examples are available.

{\bf Effect of limited label information in the target domain:} We evaluate the effect of limited label information in the target domain on mitigating the negative transfer in the static transfer learning problem. When no label information is available in the target domain, it would be difficult to characterize and avoid the negative transfer. Figure~\ref{fig:w_o_labels} shows the classification performance of transfer learning algorithms from SVHN (source) to MNIST (target) where "w/" indicates "with limited label information in the target domain" (semi-supervised transfer learning) and "w/o" indicates "without any label information in the target domain" (unsupervised transfer learning). For our proposed \model\_d algorithm, it would infer all the unlabeled target examples to produce the pseudo-labels for measuring the label-informed $\mathcal{C}$-divergence when no label information is available in the target domain.
It can be seen that without any target label information, negative transfer is more likely to occur for transfer learning algorithms. It demonstrates that limited label information in the target domain is necessary to characterize the negative transfer.

\begin{figure*}[!h]
\centering
\subfigure[DANN]{\label{fig:a}\includegraphics[width=0.323\textwidth]{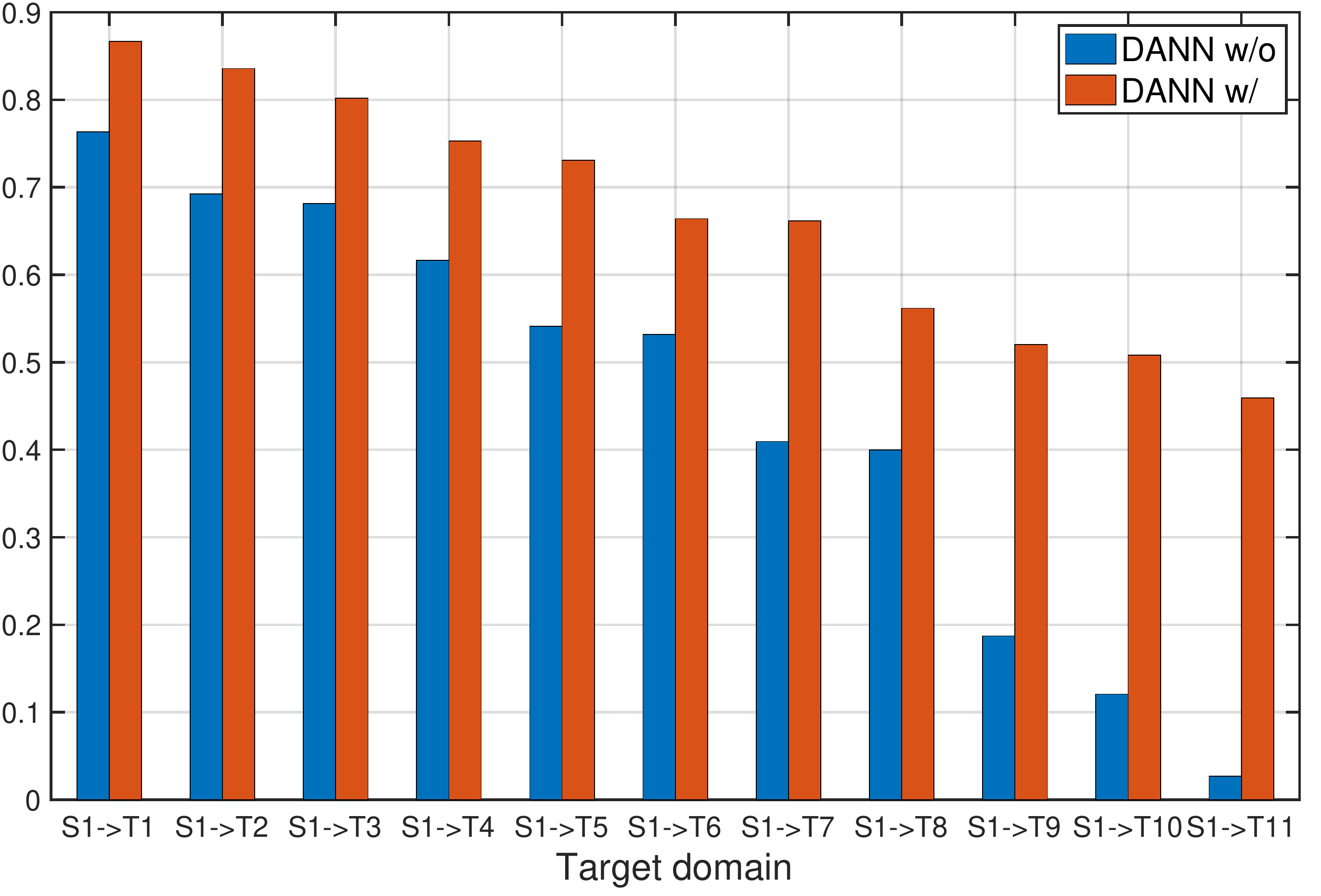}}
\subfigure[WDGRL]{\label{fig:c}\includegraphics[width=0.323\textwidth]{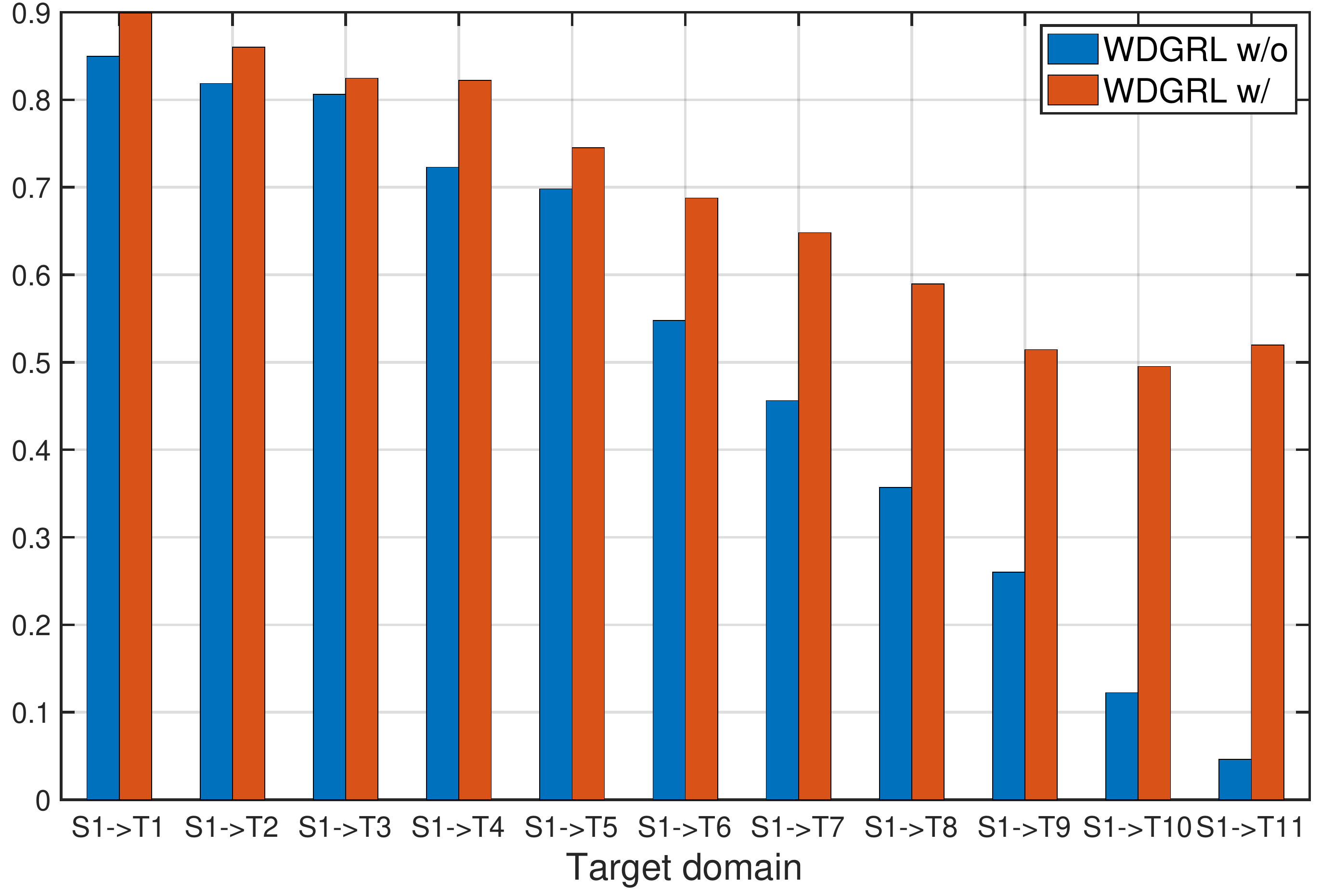}}
\subfigure[\model\_d]{\label{fig:d}\includegraphics[width=0.323\textwidth]{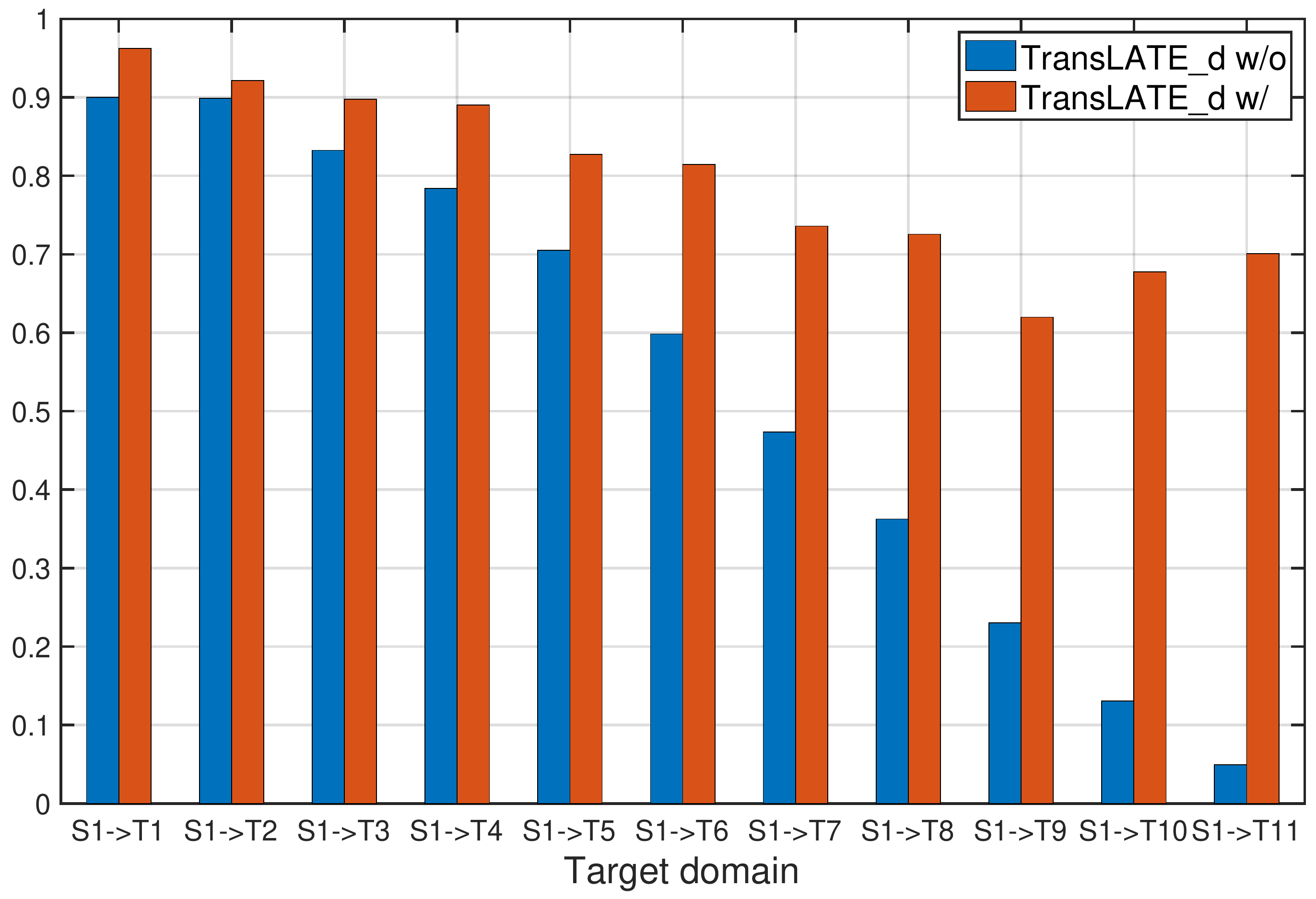}}
\vspace{-3mm}
\caption{Transfer learning accuracy with or without limited label information in the target domain}\label{fig:w_o_labels}
\end{figure*}

\begin{figure}
  \begin{center}
    \includegraphics[width=0.323\textwidth]{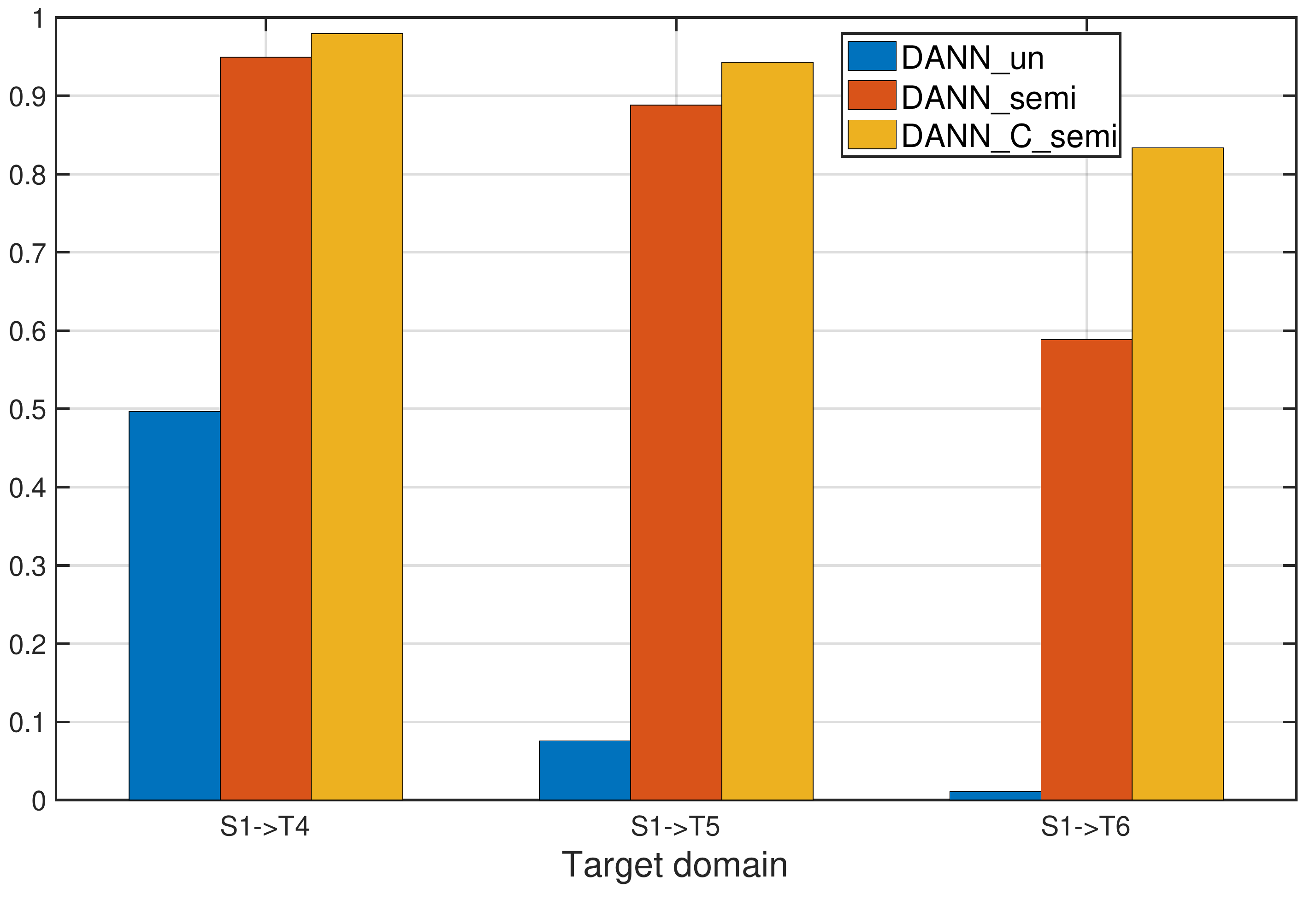}
  \end{center}
  \vspace{-5mm}
  \caption{Effect of $\mathcal{C}$-divergence}
  \label{fig:Effect_of_C}
  \vspace{-5mm}
\end{figure}
{\bf Effect of $C$-divergence:} We empirically compare the proposed $\mathcal{C}$-divergence with unsupervised domain divergence in~\cite{ganin2016domain} on the synthetic data set (shown in Figure~\ref{fig:syn_data}). To be more specific, we implement a simple domain-adversarial neural network~\cite{ganin2016domain} with either unsupervised domain divergence or our $\mathcal{C}$-divergence, and consider the following three algorithms. DANN\_un: proposed in~\cite{ganin2016domain} with unsupervised domain divergence (no labeled target examples are available); DANN\_semi: a variant of DANN\_un with unsupervised domain divergence, but with limited labeled target examples for minimizing the classification error; DANN\_C\_semi: a variant of DANN\_un with our proposed $\mathcal{C}$-divergence and limited labeled target examples could help both minimize the classification error and label-informed distribution alignment. Figure~\ref{fig:Effect_of_C} shows the transfer learning performance from the source (S1) to the target T4, T5 and T6, respectively. With limited target examples, DANN\_semi could largely avoid the negative transfer compared to DANN\_un. That confirms the effect of limited label target information for transfer learning. One intuitive explanation is that T5 and T6 (see Figure~\ref{fig:syn_data} for Target domain \#5 and \#6) are more likely to be aligned incorrectly with the source domain when no label information in the target domain is available. Limited target label information helps mitigate the occurrence of negative transfer in this case. Moreover, our proposed $\mathcal{C}$-divergence could help improve the transfer learning performance and avoid the negative transfer by encouraging the alignment of label-informed data distribution.

{\bf Visualization:}
As stated in Theorem \ref{empirical_generalization_bound}, minimizing the source error and label-informed domain discrepancy is the way to find the optimal hypothesis on minimizing the target error. Thus, the learned latent features should have the following two properties: label-informed distribution matching (minimizing $\mathcal{C}$-divergence) and highly separability (minimizing source error).
The visualization of latent feature representation is shown in Figure \ref{fig:embedding} using t-SNE~\cite{maaten2008visualizing}. It is observed that the feature representation learned by our proposed \model{} framework is well separable in the latent space and the feature distribution is matched according to the class label (10 classes in total). On the other hand, the baseline methods could not well match the label-informed data distribution though some of them (e.g., WDGRL) learned the separable latent feature representation.

\begin{figure*}[!h]
\centering
\subfigure[DANN]{\label{fig:a}\includegraphics[width=3.42cm]{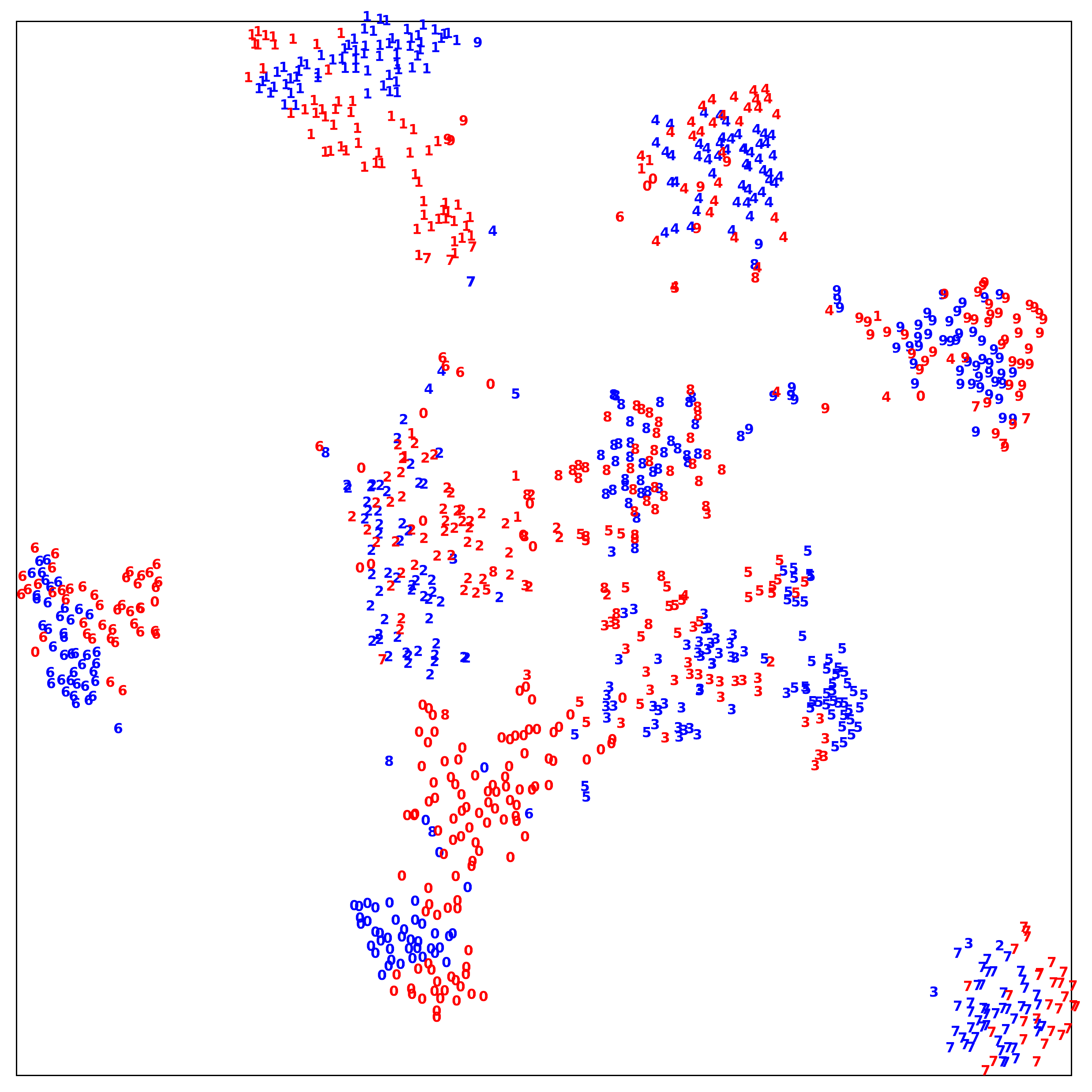}}
\subfigure[ADDA]{\label{fig:b}\includegraphics[width=3.42cm]{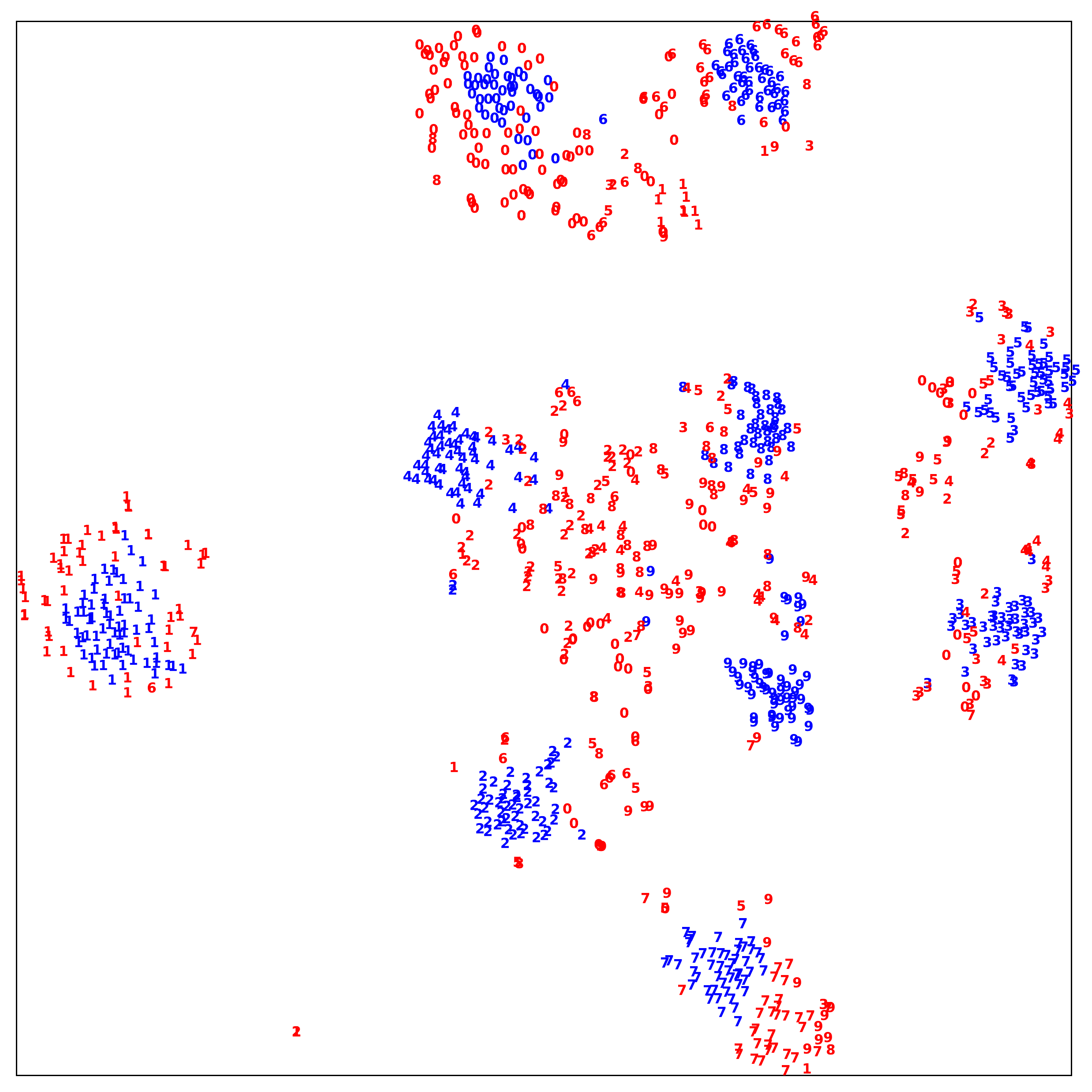}}
\subfigure[WDGRL]{\label{fig:c}\includegraphics[width=3.42cm]{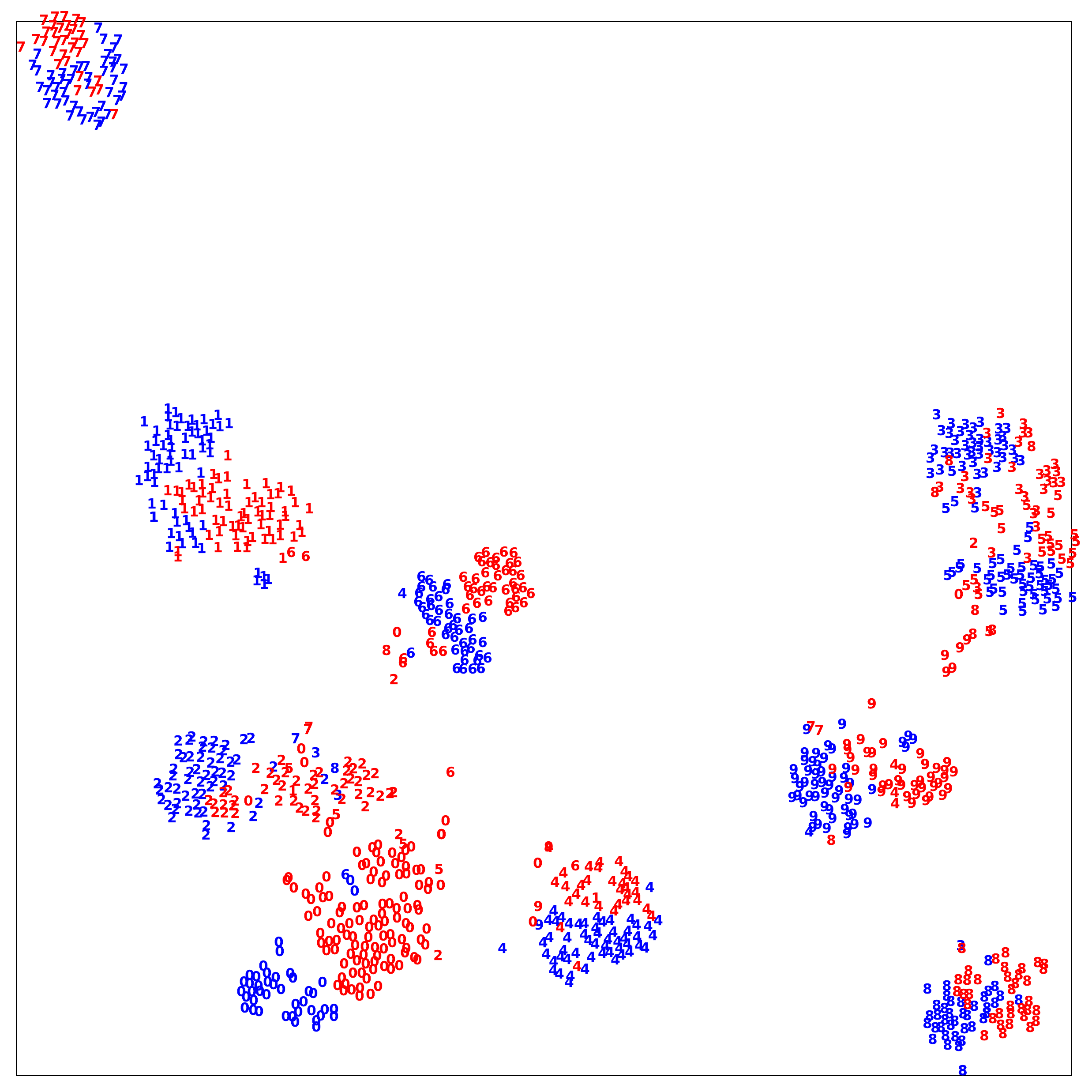}}
\subfigure[\model]{\label{fig:d}\includegraphics[width=3.42cm]{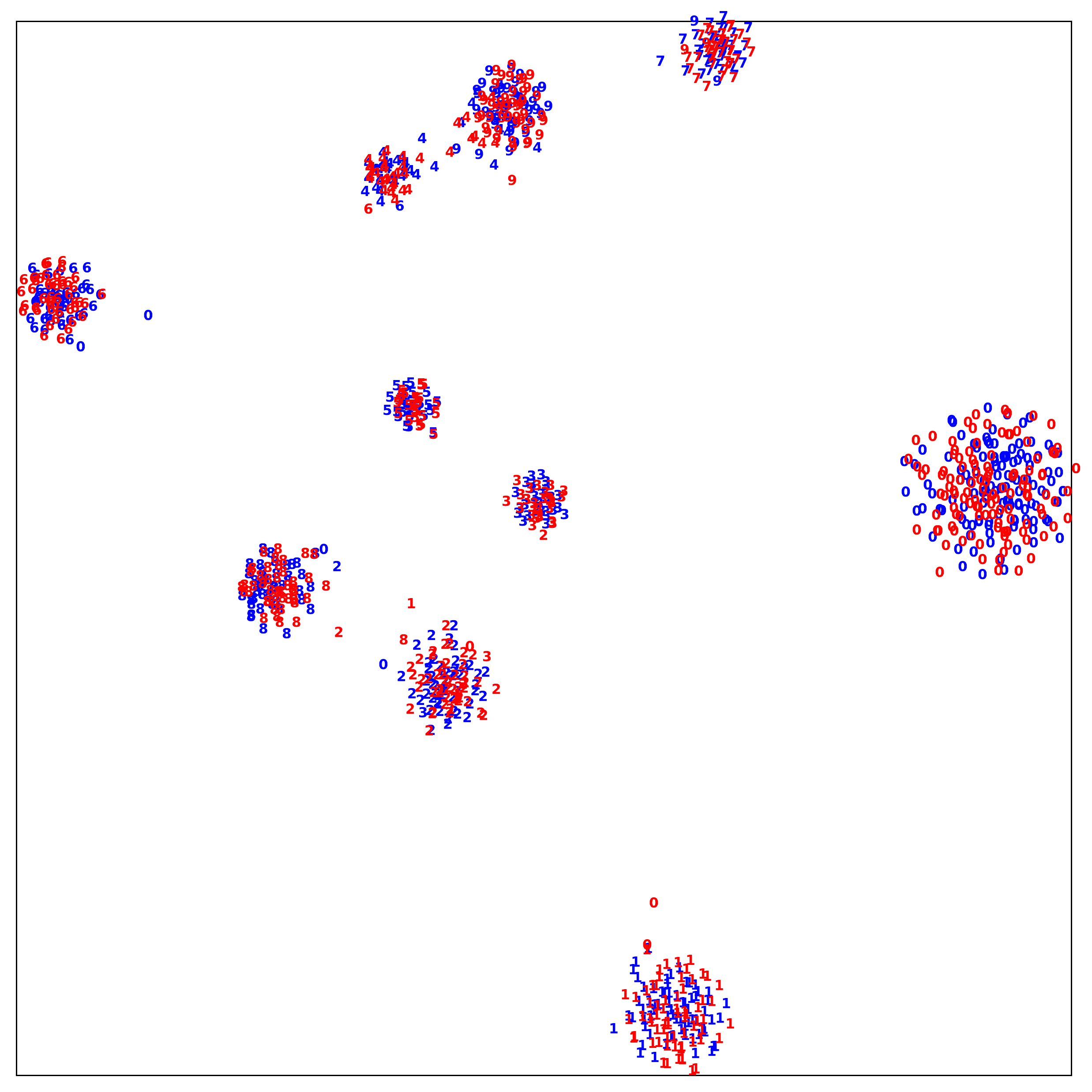}}
\vspace{-3mm}
\caption{Visualization of latent feature representation on T11 target domain (MNIST$\rightarrow$USPS). Source examples are in blue and target examples are in red. The class labels of examples are indicated in the numbers.}\label{fig:embedding}
\end{figure*}

\end{document}